\documentclass[letterpaper]{article} 
\usepackage{aaai25}  
\usepackage{times}  
\usepackage{helvet}  
\usepackage{courier}  
\usepackage[hyphens]{url}  
\usepackage{graphicx} 
\usepackage{natbib}  
\usepackage{caption} 
\usepackage{algorithm}
\usepackage{algpseudocode}

\usepackage{newfloat}
\usepackage{listings}
\DeclareCaptionStyle{ruled}{labelfont=normalfont,labelsep=colon,strut=off} 
\floatstyle{ruled}
\newfloat{listing}{tb}{lst}{}
\floatname{listing}{Listing}
\usepackage[utf8]{inputenc} 
\usepackage[T1]{fontenc}    
\usepackage{url}            
\usepackage{booktabs}       
\usepackage{amsfonts}       
\usepackage{nicefrac}       
\usepackage{microtype}      
\usepackage{amsthm,amsmath,amssymb}
\usepackage{booktabs}
\usepackage{comment}
\usepackage{enumerate}
\usepackage{mathrsfs}

\usepackage{pgfplots}
\pgfplotsset{compat=newest}
\usepackage{tikz}
\usetikzlibrary{patterns}
\usetikzlibrary{calc}

\title{Ancestral Reinforcement Learning: Unifying Zeroth-Order Optimization and Genetic Algorithms for Reinforcement Learning}
\author{
   So Nakashima\textsuperscript{\rm 1} and Tetsuya J. Kobayashi\textsuperscript{\rm 1,2}
}

\affiliations{
    \textsuperscript{\rm 1} Institute of Industrial Science, The University of Tokyo, Tokyo 153-8505, Japan;\\
    \textsuperscript{\rm 2} Universal Biology Institute, The University of Tokyo, Tokyo 113-8654, Japan. 

    naaso0510@gmail.com; tetsuya@mail.crmind.net
}

\usepackage{bibentry}

\theoremstyle{definition}
\newtheorem{theorem}              {Theorem}
\newtheorem{lemma}      [theorem] {Lemma}
\newtheorem{proposition}[theorem] {Proposition}

\newcommand{\E}{\mathbb{E}}

\newcommand{\Prob}{\mathrm{P}}

\newcommand{\popSize}{N_{\mathrm{pop}}}

\newcommand{\Aset}{\mathcal{A}}
\newcommand{\Xset}{\mathcal{X}}
\newcommand{\forward}{\Prob_{\mathrm{F}}}
\newcommand{\back}{\Prob_{\mathrm{B}}}
\newcommand{\KL}[2]{\mathcal{D}_{\mathrm{KL}}\left[#1, #2\right]}

\newcommand{\Xpath}{\mathbb{X}}
\newcommand{\Apath}{\mathbb{A}}

\newcommand{\ancestralGrad}{\hat{\nabla}_{\mathrm{ancestral}}}
\newcommand{\popV}[1]{V_{\beta}^{(#1)}}

\DeclareRobustCommand{\suptime}[2]{\ifthenelse{\equal{#2}{}}{ {#1} }{ {#1}^{(#2)} }   }

\begin{document}

\maketitle

\begin{abstract}
Reinforcement Learning (RL) offers a fundamental framework for discovering optimal action strategies through interactions within unknown environments. 
Recent advancement have shown that the performance and applicability of RL can significantly be enhanced by exploiting a population of agents in various ways.
Zeroth-Order Optimization (ZOO) leverages an agent population to estimate the gradient of the objective function, enabling robust policy refinement even in non-differentiable scenarios. As another application, genetic algorithms (GA) boosts the exploration of policy landscapes by mutational generation of policy diversity in an agent population and its refinement by selection. 
A natural question is whether we can have the best of two worlds that the agent population can have.
In this work, we propose Ancestral Reinforcement Learning (ARL), which synergistically combines the robust gradient estimation of ZOO with the exploratory power of GA. 
The key idea in ARL is that each agent within a population infers gradient by exploiting the history of its ancestors, i.e., the  ancestor population in the past, while maintaining the diversity of policies in the current population as in GA. 
We also theoretically reveal that the populational search in ARL implicitly induces the KL-regularization of the objective function, resulting in the enhanced exploration.
Our results extend the applicability of populational algorithms for RL.
\end{abstract}

\section{Introduction}

Reinforcement Learning~(RL)~\cite{sutton2018reinforcement} is a fundamental framework for discovering optimal action strategies through interactions with unknown environments, whose applications cover video games~\cite{mnih2015human}, Go~\cite{silver2016mastering}, robotics~\cite{levine2016end}, auto-drive~\cite{kendall2019learning}, and inverted helicopter flight~\cite{ng2006autonomous}.
From a broader perspective, it is a class of optimization in which the gradient of the objective function is not explicitly provided to the algorithm.
Recent advancement have shown that the performance and applicability of RL can be enhanced by exploiting a population of agents in various ways.
A representative example is Zeroth-Order Optimization (ZOO)~(also called evolutionary strategy)~\cite{rechenberg1973evolutionary, salimans2017evolution, lei2022zeroth}, in which a population of agents is used to estimate the gradient of the objective function (i.e. cumulative rewards) without explicit differentiation.
Specifically, at each iteration, a population of agents is generated by perturbing the policy of the master agent with small noises, the gradient is estimated as the average of the noises weighted by the cumulative rewards of the agents, and then the master policy is updated using the estimate. 
ZOO achieves robust optimization~\cite{lehman2018es} even for non-differentiable objective functions~\cite{kober2021from} and can be accelerated by parallel computation. 
Moreover, ZOO perturbed the policy in the parameter space instead of the action space, which enables us to simulate MDP for a long time while suppressing the variance of the estimated gradient~\cite{salimans2017evolution}

Another potential application of population is to explore a broader space of policies by keeping diversity of the agents in the population.
Genetic Algrithm~(GA)~\cite{such2017deep, risi2019deep, whitley1993genetic} has been leveraged for this purpose; at each iteration, a new population of agents is generated from the parent population via  mutation and/or crossover of the parent agents for exploring more diverse policies than single agent. Then, the new population is shaped by selecting agents with higher cumulative rewards.
Compared with ZOO which retains only one policy at each update, GA maintains multiple policies, enabling GA more exploratory than ZOO.
However, GA does not estimate the gradient, making it inefficient for certain classes of problems~\cite{such2017deep}.

This work aims at combining the best of these two approaches by keeping the variety of policies in a population for exploration whilst enabling each agent estimates the gradient of the objective function without differentiation.
The difficulty lies in the conflict in the ways that the two approaches utilize population; ZOO amalgamates population whereas GA has to keep a diverse population.
In previous attempts to combine GA and the gradient update~\cite{khadka2018evolution, callaghan2023evolutionary}, an "elite" agent was updated by utilizing the history of a population of other agents that follow GA.
While its efficiency was verified empirically, it lacks a theoretical basis for underlying mechanism and performance.

In this paper, we propose \textbf{Ancestral Reinforcement Learning} (ARL), which combines ZOO and GS. In ARL, the gradient is estimated, without generating a new population like ZOO, by utilizing the survivor-ship bias of the ancestor populations of individual agents, and thus the current agent population can retain diversity for exploration. 
We demonstrate the effectiveness of ARL by numerical experiments.
Our major contribution in this work is more theoretical than experimental.
Specifically, we prove that each agent in ARL can estimate the gradient of objective function by utilizing survivorship bias. In addition, we show that the populational search in ARL implicitly induces the KL-regularization of the objective function, resulting in the enhanced exploration of ARL compared to ZOO.
This theoretical basis would contribute to devising new population algorithms and verifying their efficiency.

\subsection{Related Works}
Populations of agents have been employed in several ways for solving various optimization problems, e.g., GA~\cite{such2017deep, risi2019deep,whitley1993genetic}, ZOO~\cite{khadka2018evolution, lei2022zeroth}, and augmented random search~\cite{mania2018simple, kober2021from}.

Among others, the combination of evolutionary algorithms with learning was pioneered by Hinton and Nowlan~\cite{hinton1987learning} inspired by the Baldwin effect in evolutionary biology.
They attempted to leverage additional random search of individual agents in GA.
In this paper, we are interested in integrating populational search of GA with gradient estimate of individual agents via survivorship bias without differentiation of objective function.
Recent papers~\cite{khadka2018evolution, callaghan2023evolutionary} attempted to combine GA and the gradient method.
In these algorithms, the GA algorithm yields histories of actions and states. Then, one agent outside of GA learns from those histories by the gradient method.
However, the theoretical foundation for why these algorithms work effectively is still unclear.
This paper reveals the meaning of learning from the paths generated by GA.

\section{Preliminary}
In this section, we outline the theoretical basis of RL, entropy-regularized RL, ZOO, and GA before introducing ARL.

\subsection{Reinforcement Learning}
We will consider Markov Decision Process (MDP) with state space $\Xset$ and action space $\Aset$~\cite{puterman2014markov}. 
In this paper, we denote the state and action at time $t$ by $\suptime{x}{t}\in \Xset$ and $\suptime{a}{t}\in \Aset$, respectively.
When the state is $\suptime{x}{t}$, an agent chooses its action $\suptime{a}{t}$ by a policy $\pi(\cdot \mid \suptime{x}{t})$, which is a distribution on $\Aset$ conditioned by the state $\suptime{x}{t}$.
In this main text, we consider the case where the next state $\suptime{x}{t+1}$ is determined deterministically given the current state $\suptime{x}{t}$ and the action $\suptime{a}{t}$ as $\suptime{x}{t+1} = f(\suptime{x}{t}, \suptime{a}{t})$ for simplicity.
In addition, $\suptime{x}{0}$ is assumed to be constant.
The result for general MDP is shown in Appendix.
We denote the history of the states $\{\suptime{x}{s}\}_{s=0,1,\dots}$ and that of the actions $\{\suptime{a}{s}\}_{s=0,1,\dots}$ by $\Xpath$ and $\Apath$, respectively.
We also define the truncated histories as $\suptime{\Xpath}{t:t'} := \{\suptime{x}{s}\}_{s=t,t+1,\dots, t'}$ and $\suptime{\Apath}{t:t'} := \{\suptime{a}{s}\}_{s=t,t+1,\dots, t'}$.
When $t=0$ and $t'=\infty$, we omit the superscript and just denote as $\Xpath$. Similarly, we use $\suptime{\Xpath}{t:}$ for $t'=\infty$.
The discounted cumulative reward $R[\Xpath, \Apath]$ is defined by $\sum_{t} \gamma^t r(\suptime{x}{t}, \suptime{a}{t})$, where $\gamma \in (0,1)$ is a discount factor and $r$ is a reward function.
The objective function of RL is the expected value of the cumulative reward function defined by $J(\pi) := \E_{\forward^{\pi}}\left[ R[\Xpath, \Apath]\right]$ where 
\begin{align}
    \label{eq:def-forward}
    \forward^{\pi}[\Xpath, \Apath] :=  \prod_{t} \pi(\suptime{a}{t} \mid \suptime{x}{t}) \delta_{\suptime{x}{t+1}, f(\suptime{x}{t}, \suptime{a}{t})},
\end{align}
is the probability that policy $\pi$ generates the history pair, $\Xpath$ and $\Apath$.

\subsection{Entropy-regularized Reinforcement Learning}
To improve the robustness of learning and exploration of the policy, a regularizer is often appended to $J(\pi)$.
An example is the entropy-regularized RL~\cite{howard1972risk, sutton2018reinforcement, ziebart2010modeling, haarnoja2018soft, braun2011path, neu2017unified, fox2015taming} in which the objective function $J_{\mathrm{ent}}(\pi) = \E_{\forward^{\pi}}\left[ R_{\mathrm{ent}}[\Xpath, \Apath]\right]$ is defined with 
\begin{align}
    \label{eq:max-ent-rl}
    R_{\mathrm{ent}}[\Xpath, \Apath] &= \sum_t \gamma^t \left( r(\suptime{x}{t}, \suptime{a}{t}) + \frac{1}{\beta} \log \pi( \suptime{a}{t} \mid \suptime{x}{t}) \right).
\end{align}
The extra term $\log \pi( \suptime{a}{t} \mid \suptime{x}{t})$ is related to the Kullback-Leibler (KL) divergence and encourages the agents to be more exploratory.
 For the value function defined as  
\begin{align}
    V(x) = \E_{\forward^{\pi}}\left[ R_{\mathrm{ent}}[\Xpath, \Apath] \mid \suptime{x}{0} = x \right],
\end{align}
a Bellman-type recursive equation is satisfied:
\begin{align}
    \label{eq:max-ent-v-bellman}
    V(x) = \E_{\pi(a \mid x)} \left[r(x, a) +  \frac{1}{\beta} \log \pi(a \mid x) + \gamma  V(f(x,a))\right].
\end{align}
Then, the objective is obtained by $J_{\mathrm{emt}} = \E[V(x)]$.

\subsection{Zeroth-Order Optimization}

Algorithm~\ref{alg:es} describes how ZOO works, in which we iteratively update a single master policy $\pi_{\suptime{\theta}{n}}$. The policy $\pi_{\theta}$ is parametrized by $\theta$. 
At each iteration $n$, we generate a population of policies from the master, the $i$th policy of which has parameter $\suptime{\theta_i}{n} = \suptime{\theta}{n} + \sigma \suptime{\epsilon}{i}$ perturbed by small noise $ \sigma \suptime{\epsilon}{i}$.
Here, each element of $\suptime{\epsilon}{i}$ follows the standard normal distribution independently, and the parameter $\sigma$ controls the size of noise. 
We then observe the cumulative rewards $R_i$ of each agent by running a simulation.
The gradient $\nabla_\theta J(\pi_{\suptime{\theta}{i}})$ is estimated by $\suptime{\hat{g}}{n} := \sum_i R_i \suptime{\epsilon}{i} / \sigma $.
Using this estimation, we update the master policy by the gradient ascent, that is, 
\begin{align}
    \suptime{\theta}{n+1} \gets \suptime{\theta}{n} + \alpha \suptime{\hat{g}}{n},
\end{align}
where $\alpha$ is a learning rate.
Previous studies~\cite{rechenberg1973evolutionary, salimans2017evolution} have proven that $\suptime{\hat{g}}{n}$ is an unbiased estimator of the gradient:
$\E[\suptime{\hat{g}}{n}] = \nabla_{\theta} J(\pi_{\suptime{\theta}{n}})$.
It should be noted that zeroth-order optimization in a broader context is not limited to this specific instance~\cite{nesterov2017random, flaxman2005online, ghadimi2013stochastic, douchi2015optimal}.

\begin{algorithm}[tb]
\caption{Zeroth Order Optimization}
\label{alg:es}
\begin{algorithmic}[1]
\State{Sample an initial policy $\pi_{\suptime{\theta}{0}}$.}
\For{$n=0, 1,\dots$}
        \For{$i = 0,1,\dots, \popSize -1$}
            \State{Sample noise $\suptime{\epsilon}{i}$.}
            \State{$\suptime{\theta_i}{n} \gets \suptime{\theta}{n} + \sigma \suptime{\epsilon}{i}$.}
            \State{Run a simulation with policy $\pi_{\suptime{\theta_i}{n}}$ and observe discounted cumulative reward $R_i$.}
        \EndFor
        \State{Estimate gradient by $\suptime{\hat{g}}{n} \gets \sum_i R_i \suptime{\epsilon}{i} / \sigma $. }
        \State{$\suptime{\theta}{n+1} \gets \suptime{\theta}{n} + \alpha\suptime{\hat{g}}{n} $.}
\EndFor
\end{algorithmic}
\end{algorithm}

\subsection{Population Optimization via GA (POGA)}
GA is a metaheuristic inspired by biological evolution and consists of various ingredients~\cite{eiben2015introduction}. 
For comparison and integration with ZOO, we here focus only on its restricted aspect as a population optimization algorithm using selection and random mutation.
Population optimization via GA (POGA) described in Algorithm~\ref{alg:ga} updates a population of policies $\{\pi_{\suptime{\theta_i}{n}}\}_{i=0,1,\dots, \popSize-1}$.
At each iteration $n$, small noise is added to the parameter of each policy, resulting in a perturbed population $\{\pi_{\suptime{\theta_i'}{n+1}}\}_{i=0,1,\dots, \popSize-1}$.
This process is called mutation.
The cumulative reward $R_i$ for each policy is obtained by running a simulation. 
Then, a new population $\{\pi_{\suptime{\theta_i}{n+1}}\}_{i=0,1,\dots, \popSize-1}$ for the next step is formed by independent sampling of policies from the mutated population with a probability proportional to fitness $\exp(\beta R_i)$, where $\beta$ is a hyperparameter.
Here, the fitness is defined so that the algorithm becomes invariant to the transformation of the rewards of the form: $r'(x,a) = r(x,a) + c$, where $c$ is a constant.

\begin{algorithm}[tb]
\caption{Population Optimization via GA (POGA)}
\label{alg:ga}
\begin{algorithmic}[1]
\State{Sample initial policies $\{\pi_{\suptime{\theta_i}{0}}\}_{i=1,2,\dots,\popSize-1}$.}
\For{$n=0, 1,\dots$}
        \For{$i = 0,1,\dots, \popSize -1$}
            \State{Mutate policy $\pi_{\suptime{\theta_i}{n}}$ by adding small noise to the parameter and obtain $\pi_{\suptime{\theta_i'}{n}}$.}
            \State{Run simulation for the mutated policy $\pi_{\suptime{\theta_i'}{n}}$ and observe the discounted cumulative reward $R_i$.}
            \State{Calculate fitness by $\suptime{f}{n}_i \gets \exp(\beta R_i)$.}
        \EndFor
        \State{Select each agent $\pi_{\suptime{\theta_{i}}{n+1}}$ in the next population independently from the mutated population $\{\pi_{\suptime{\theta_i'}{n}}\}_{i =0,1,\dots, \popSize-1}$ with probability proportional to $\suptime{f_i}{n}$.}
\EndFor
\end{algorithmic}
\end{algorithm}

The principle of Algorithm~\ref{alg:ga} for optimization becomes explicit by considering the large size limit $\popSize \to \infty$.
Suppose here that the mutation does not occur in Algorithm~\ref{alg:ga} because selection plays the definite role in defining optimization whereas mutation works for exploration.
Let $\suptime{p}{n}(\pi)$ be the frequency of the policy $\pi$ in the population at the $n$-th iteration.
For $\popSize \to \infty$, we can approximate the time evolution of $\suptime{p}{n}$ by
\begin{align}
    \label{eq:def-pd}
    \suptime{p}{n+1}(\pi) =   \frac{\E_{ \forward^{\pi}}\left[\exp(\beta R[\Xpath, \Apath])\right]}{\E_{\suptime{p}{n}(\pi')}\left[\E_{ \forward^{\pi'}}\left[\exp(\beta R[\Xpath, \Apath])\right] \right]} \suptime{p}{n}(\pi). 
\end{align}
The numerator $\E_{ \forward^{\pi}}\left[\exp(\beta R[\Xpath, \Apath])\right]$ is the objective-dependent weighting for sampling the policy $\pi$ by selection.
In the context of population dynamics, it can be seen as the expected number of daughters that the parent with policy $\pi$ generates.
The logarithm of this quantity
\begin{align}
    \label{eq:population_fitness}
    \lambda(\pi) := \frac{1}{\beta} \log\E_{\forward^{\pi}}\left[ \exp({\beta R[\Xpath, \Apath])}\right],
\end{align}
is known as the population fitness of policy $\pi$.
From the time evolution defined in~\eqref{eq:def-pd}, $\lambda(\pi)$ is proportional to the exponential growth rate of the fraction of policy $\pi$ in the population. From this, we can see that the maximizer of $\lambda(\pi)$ dominates the population:
\begin{lemma}
    \label{thm:population-fitness}
    Assume that Algorithm~\ref{alg:ga} does not mutate the policies.
    If $\pi$ satisfies $\lambda(\pi) > \lambda(\pi')$ for any other $\pi'$ and the population size is large enough, then $\pi$ dominates the population.
\end{lemma}
This lemma means that $\lambda(\pi)$ rather than $J(\pi)$ is the objective function of POGA. Nevertheless, population fitness can be seen as a generalization of the usual objective function $J(\pi)$.
In fact, $\lambda(\pi)$ converges to $J(\pi)$ as $\beta \to 0$.

\section{Ancestral Reinforcement Learning: Unification of ZOO and POGA}

\begin{figure}[tb]
  \begin{center}
        \includegraphics[width=1.0\linewidth]{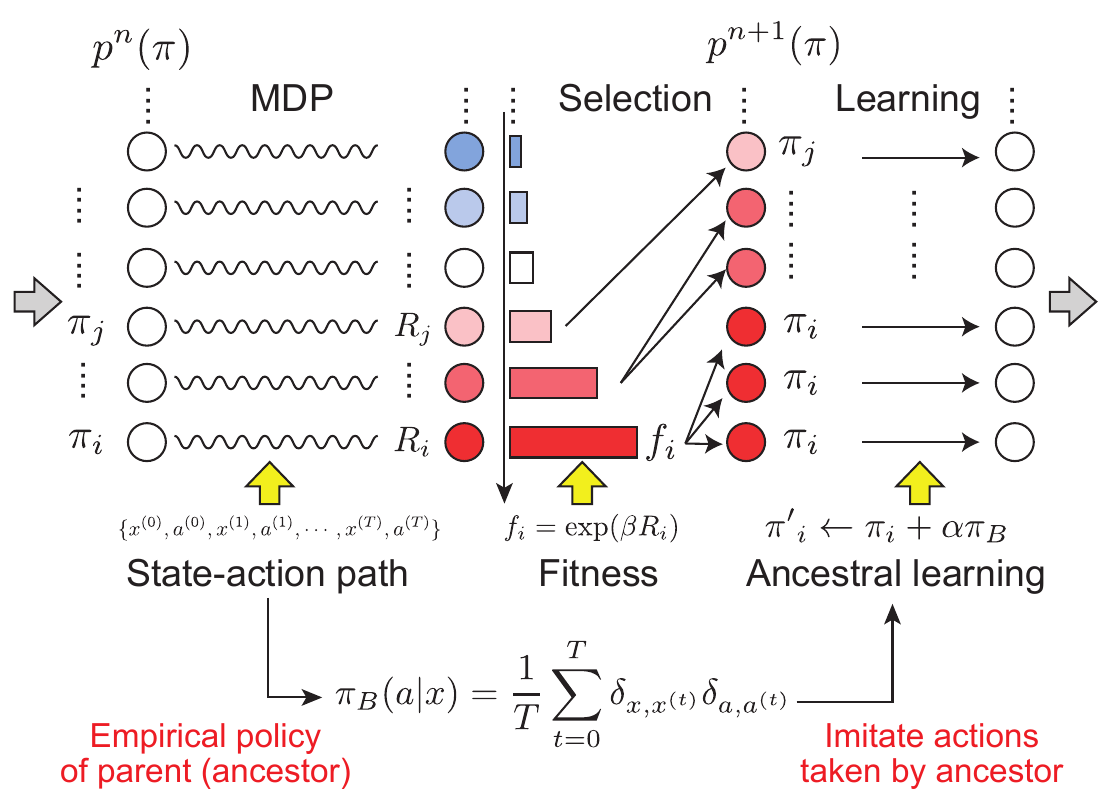}
   \end{center}
    \caption{Schematic representation of Ancestral Reinforcement Learning (ARL). In ARL, agents of the next generation is selected from the current population according to their fitness defined by cumulative rewards observed by MDP simulation. After selection, the policy of each agent is updated by ancestral learning. 
    At ancestral learning step, the policy is modified to imitate what the ancestor did using the empirical policy of the parent (ancestor). Owing to the survivorship bias, this update effectively works as a kind of gradient ascent.
    }
    \label{fig:ancestral_learning}
\end{figure}

If the mutation step in Algorithm~\ref{alg:ga} can be replaced with gradient ascent, we may obtain a more efficient algorithm than POGA that unifies ZOO and POGA.
The core difficulty is how to estimate the gradient. 
In POGA, policies in the population $\pi_{\suptime{\theta_i}{n}}$ are derived from different parent policies whereas they are from the same master policy in ZOO. 
Thus, we cannot use the population of POGA for gradient estimation as in ZOO.

To resolve this difficulty, we propose \textbf{Ancestral Reinforcement Learning} (ARL) (Algorithm~\ref{alg:arl} and Figure~\ref{fig:ancestral_learning}).
In ARL, the gradient is estimated by ancestral learning (Algorithm~\ref{alg:ancestral_learning}), in which the population of ancestors for each policy is used to estimate gradient without generating new population. 
In this work, we focus on ARL using only the information of parent, i.e., the ancestor of one-generation ago while we could employ the information of earlier ancestor population.
Specifically, we replace the random policy mutation in Algorithm~\ref{alg:ga} at step 4 with an update rule inducing the policy to repeat the parent's actions, i.e. the actions learned from the ancestor.
For an illustrative case of tableau MDP, that is, $\theta = \{\pi(a \mid x)\}_{a \in \Aset, x \in \Xset}$, the empirical distribution $\pi_B$ of the parent's actions of the $i$-th policy at the $n$-th generation is defined by
\begin{align}
    \suptime{\pi}{n}_B(x, a) := \frac{1}{T} \sum_{t=0}^{T-1} \delta_{x, \suptime{x_{p(i)}}{t}} \delta_{a,\suptime{a_{p(i)}}{t}}, 
\end{align}
where $\suptime{x_{p(i)}}{t}$ and $\suptime{a_{p(i)}}{t}$ are the state and action of the parent $p(i)$ of the $i$-th policy at time $t$, $\delta_{\cdot, \cdot}$ is Kronecker's delta function, and $T$ is the time length of the MDP simulation for one generation.
Then, we replace the mutation in Algorithm~\ref{alg:ga} with the following update using $\suptime{\pi}{n}_B$: 
\begin{align}
    \label{eq:ancestral_learning}
    \suptime{\pi'_i}{n}(a \mid x) :\propto \suptime{\pi_i}{n}(a \mid x)+ \alpha \suptime{\pi_B}{n}(a \mid x).
\end{align}
The updated policy is the mixture of the original policy $\suptime{\pi_i}{n}$ and the frequency of actions taken by the parent $\suptime{\pi_B}{n}$.

An intuition why mimicking parent's actions leads to a gradient estimation can be gained by considering the simplest situation with single state and two actions, i.e., $\Xset = \{x_*\}$ and $\Aset = \{a_0,a_1\}$. Then, the reward is determined solely by action as $r(a_0) = w$ and $r(a_1) = 0$, where $w > 0$.
Without selection via fitness, imitating parent does not benefit on average because it is uncertain whether the parent chose good actions or not. 
With selection, in contrast, the policies who have higher fitness are overrepresented in the population. 
The parents of those policies are statistically biased to those who chose better action $a_0$ more frequently than others. 
Owing to this surviorship bias, the empirical distribution of parent's action $\pi_B$ works as the gradient towards the better policy.
In the next section, we will make this augment more rigorous and general by show that the update, \eqref{eq:ancestral_learning}, is indeed a stochastic gradient ascent.
Notice that AL is similar to online expert algorithms~\cite{Cesa-Bianchi2007improved} where the policy is modified toward those of the winning players, yet AL can achieve it without knowing who were the winners via survivorship bias.

The ancestral learning can be extended to general MDP.
From Lemma~\ref{thm:population-fitness}, the objective of ARL is $\lambda(\pi)$, which generalizes $J(\pi)$. 
Let the path-wise empirical distribution of parent's states and actions for $i$th policy be 
\begin{align}
    j_{i}[\Xpath, \Apath] = \prod_t \delta_{\suptime{x}{t}, \suptime{x_{p(i)}}{t}} \delta_{\suptime{a}{t}, \suptime{a_{p(i)}}{t}}.
\end{align}
The gradient of the objective function is estimated as
\begin{align}
    \frac{\partial \lambda(\pi_{\suptime{\theta_i}{n}})}{\partial \suptime{\theta_i}{n}} 
    &\approx  \E_{j_{i}[\Xpath, \Apath]}\left[\sum_{t=0}^{T} \nabla_{\suptime{\theta_i}{n}} \log \pi_{\suptime{\theta_i}{n}}(\suptime{a}{t} \mid \suptime{x}{t})\right] \notag \\
    &\propto \sum_{t=0}^{T} \nabla_{\suptime{\theta_i}{n}} \log \pi_{\suptime{\theta_i}{n}}(\suptime{a_{p(i)}}{t} \mid \suptime{x_{p(i)}}{t}),\notag \\
    &=: \ancestralGrad \lambda(\pi_{\suptime{\theta_i}{n}}),
    \label{eq:ancestral_grad}
\end{align}
which we call an ancestral estimator of the gradient.
The ancestral learning for general MDP is obtained as
\begin{align}
    \label{eq:ancestral-learning-theta}
    \suptime{\theta'_i}{n} \gets \suptime{\theta_i}{n} + \alpha  \ancestralGrad \lambda(\pi_{\suptime{\theta_i}{n}}).
\end{align}
Notice that this formula is similar to the estimation of the gradient in the actor-critic algorithm~\cite{sutton2018reinforcement}.
The difference is that our formula is weighted by the survivorship bias via the empirical distribution of the parent $j_{i}[\Xpath, \Apath]$, while the actor-critic algorithm is weighted by an estimator of the advantage function. 
We also note that when we consider the natural gradient ascent with the estimated gradient~\eqref{eq:ancestral_grad} for tableau MDP, we have the same update as~\eqref{eq:ancestral_learning} (See Appendix for the proof).

\begin{algorithm}[tb]
\caption{Ancestral Reinforcement Learning}
\label{alg:arl}
\begin{algorithmic}[1]
\State{Sample initial policies $\{\suptime{\pi_{\theta_i}}{0}\}_{i =1,2,\dots, \popSize-1}$.}
\For{$n =0,1,\dots$}
        \For{$i = 0,1,\dots, \popSize -1$}            
            \State{ $\suptime{\theta_i'}{n} \gets \mathrm{AncestralLearning}(\suptime{\theta_i}{n}, \Xpath, \Apath)$, where $\Xpath$ and $\Apath$ are the parent's history of states and action (Skip this step at $n=0$).}
            \State{Run simulation until time $T$ with policy $\pi_{\suptime{\theta_i'}{n}}$ and observe discounted cumulative reward $R_i$, the history $\Xpath = \{\suptime{x}{0}, \suptime{x}{1}, \dots \}$ of the states, and that of $\Apath = \{\suptime{a}{0}, \suptime{a}{1}, \dots \}$ of actions.}
        \EndFor
        \State{Calculate fitness as $\suptime{f}{n}_i \gets \exp(\beta R_i)$.}
        \State{Select each agent $\pi_{\suptime{\theta_{i}}{n+1}}$ for the next population independently from the mutated population $\{\{\pi_{\suptime{\theta_i'}{n}}\}_{i =0,1,\dots,\popSize-1}$ with probability proportional to $\suptime{f_i}{n}$.}
\EndFor
\end{algorithmic}
\end{algorithm}

\begin{algorithm}[tb]
\caption{Ancestral Learning}
\label{alg:ancestral_learning}
\begin{algorithmic}[1]
\Require{Parameter $\suptime{\theta_i}{n}$ and parent's history of states $\Xpath = \{\suptime{x}{t}\}_{t=0,1,\dots} $ and actions $\Apath = \{\suptime{a}{t}\}_{t=0,1,\dots}$.}
\Ensure{Updated parameter $\suptime{\theta_i'}{n+1}$.}
\State{Calculate the ancestral estimator of the gradient $\ancestralGrad\lambda(\pi_{\suptime{\theta_i}{n}})$ by Equation~\eqref{eq:ancestral_grad}.}
\State{$\suptime{\theta_i'}{n+1} \gets \suptime{\theta_i}{n} + \alpha \ancestralGrad(\pi_{\suptime{\theta_i}{n}})$.}
\end{algorithmic}
\end{algorithm}

\section{Theoretical Basis of ARL}
In this section, we reveal the relationship between the ancestral estimator of the gradient~\eqref{eq:ancestral_grad} and the true gradient through the following two steps:
\begin{enumerate}
    \item We prove that the ancestral estimator of gradient~\eqref{eq:ancestral_grad} is an unbiased estimator of $\nabla_{\theta} \lambda(\pi)$ (Theorem~\ref{thm:ancestral-gradient}).
    \item We clarify that $\lambda(\pi)$ can be interpreted as the objective function $J(\pi)$ with KL regularizer (Theorem~\ref{thm:bellman-back}). 
\end{enumerate}
Overall, it will be shown that ARL optimizes $J(\pi)$ with KL regularization and that ancestral learning is interpreted as a stochastic gradient ascent for this objective function.
Thus, ARL is indeed an unification of POGA and ZOO.

\subsection{Step 1: AL is grandient ascent for popluation fitness}
\label{subsec:ancestral-learning-and-gradient}
We prove the following relationship between ancestral learning and population fitness.
\begin{theorem} \label{thm:ancestral-gradient}
The ancestral estimator $\ancestralGrad \lambda(\pi)$ is propotional to an unbiased estimator of the  gradient of $\lambda(\pi)$:
\begin{align}
    \E[\ancestralGrad \lambda(\pi_{\suptime{\theta_i}{n}})] \propto \nabla_{\suptime{\theta_i}{n}} \lambda(\pi_{\suptime{\theta_i}{n}}), 
\end{align}
\end{theorem}
This theorem has been proven for a specific Markov process~\cite{nakashima2022acceleration}, and we generalize the result for MDP.

To prove this theorem, let us first characterize the expectation of $j_{i}[\Xpath, \Apath]$ in the definition of $\ancestralGrad \lambda(\pi_{\suptime{\theta_i}{n}})$.
If the population size is sufficiently large, we have sufficient number of members in the population, whose parents have the same policy $\pi$ as the $i$th one and generated the same state action history, $j_{\pi}[\Xpath, \Apath]=j_{i}[\Xpath, \Apath]$.
Since the fitness is defined by $\exp(\beta R[\Xpath, \Apath])$, the conditional probability that history $j_{\pi}[\Xpath, \Apath]$ is observed becomes:
\begin{align}
    \label{eq:def-back}
    \back^{\pi}[\Xpath, \Apath] :=\E\left[ j_{\pi}[\Xpath, \Apath] \right] \propto e^{\beta R[\Xpath, \Apath]} \forward^{\pi}[\Xpath, \Apath],
\end{align}
 (See Appendix for the proof).
The probabilities, $\forward$ and $\back$, are called forward and  backward probabilities, respectively.

Next, we differentiate the population fitness.
We can prove a generalization of the policy gradient theorem by direct calculation.
\begin{proposition}
\label{thm:theta-gradient-pop}
    \begin{align}
        \nabla_{\theta}\lambda(\pi_{\theta}) =\frac{1}{\beta} \E_{\back^{\pi_{\theta}}[\Xpath, \Apath]} \left[ \sum_{t=0}^{T} \nabla_\theta \log \pi_{\theta}(\suptime{a}{t}, \suptime{x}{t})  \right].
    \end{align}
\end{proposition}
See Appendix for the proof.
By combining these two results, we have Theorem~\ref{thm:ancestral-gradient}.
It should be noted that $\ancestralGrad \lambda(\pi_{\suptime{\theta_i}{n}})$ without expectation could also be a good approximation for $\nabla_{\suptime{\theta_i}{n}} \lambda(\pi_{\suptime{\theta_i}{n}})$ if $T$ is large because of the law of large number for $T$.

\subsection{Step 2: Implicit KL regularization}
We establish a connection of population fitness $\lambda(\pi)$ with the original objective function $J(\pi)$.
To see this, we introduced a generalization of the value function for population fitness.
We define a generalized $V$-function at time $t$ by
\begin{align}
    &\popV{t}(x) := \frac{1}{\beta} \log \E_{\forward^{\pi}[\cdot\mid \suptime{x}{t}=x]} \left[ \exp\left(\beta \suptime{R}{t:}[\suptime{\Xpath}{t:},\suptime{\Apath}{t:}] \right) \right],\\
    &\suptime{R}{t:}[\suptime{\Xpath}{t:},\suptime{\Apath}{t:}] :=  \sum_{s\ge t} \gamma^{s} r(\suptime{x}{s}, \suptime{a}{s}),
\end{align}
where $x$ is the state at time $t$.
From the definition, we can easily see that $\lambda(\pi) = \popV{0}(\suptime{x}{0})$.

Since the definition of $\popV{t}$ is in the logarithmic sum-exp form, its recursive equation cannot be obtained in a usual way.
Nonetheless, we can derive a Bellman-type recursive equation for $\popV{t}$ using the following variational representation:
\begin{align}
    \label{eq:variational-representation}
    \popV{t}(x)&= \max_{\mathrm{P}[\Xpath, \Apath]}\left\{\E_{\mathrm{P}}\left[\suptime{R}{t:}[\suptime{\Xpath}{t:},\suptime{\Apath}{t:}]\right]\right.\\
    - \frac{1}{\beta} &\left.\KL{\mathrm{P}[\suptime{\Xpath}{t+1:},\suptime{\Apath}{t:} \mid \suptime{x}{t}]}{\forward[\suptime{\Xpath}{t+1:},\suptime{\Apath}{t:} \mid \suptime{x}{t}]}\right\},  
\end{align}
where $\mathrm{P}$ runs over all probability distributions for $\Xpath$ and $\Apath$, and $\KL{\cdot}{\cdot}$ is the KL-divergence.
The maximizer is 
\begin{align}
    \label{eq:truncated-back}
    &\back[\suptime{\Xpath}{t+1:},\suptime{\Apath}{t:} \mid \suptime{x}{t}] :=  \frac{\back[\Xpath, \Apath]}{\sum_{\suptime{\Xpath}{t+1:}, \suptime{\Apath}{t:}} \back[\Xpath, \Apath] }\\
    &\propto \exp(\beta \suptime{R}{t:}[\suptime{\Xpath}{t:},\suptime{\Apath}{t:}]) \forward[\suptime{\Xpath}{t+1:},\suptime{\Apath}{t:} \mid \suptime{x}{t}].
\end{align}
For the derivation of this explicit formula, see Appendix.
Since the maximizer is $\back[\suptime{\Xpath}{t+1:},\suptime{\Apath}{t:} \mid \suptime{x}{t}]$, we have
\begin{align}
    \label{eq:pop-v-by-back-prob}
    \popV{t}&(x)= \E_{\back}\left[\suptime{R}{t:}[\suptime{\Xpath}{t:},\suptime{\Apath}{t:}]\right]\\
    &-\frac{1}{\beta}\KL{\back[\suptime{\Xpath}{t+1:},\suptime{\Apath}{t:} \mid \suptime{x}{t}]}{\forward[\suptime{\Xpath}{t+1:},\suptime{\Apath}{t:} \mid \suptime{x}{t}]}.
\end{align}
Using this expression, we obtain a recursive equation:
\begin{theorem}
\label{thm:bellman-back}
The generalized V-function $\popV{t}(\suptime{x}{t})$ satisfies the following Bellman-type equation with KL regularization:
  
\begin{align}
        \label{eq:bellman-back}
    \popV{t}(\suptime{x}{t}) &=
    \E_{\back(\suptime{a}{t}\mid \suptime{x}{t})}\left[\gamma^{t} r(\suptime{x}{t}, \suptime{a}{t}) \right.\\
    &\left.-\frac{1}{\beta}\log \frac{\back(\suptime{a}{t}\mid \suptime{x}{t})}{\pi(\suptime{a}{t}\mid \suptime{x}{t})} + \popV{t+1}(f(\suptime{x}{t}, \suptime{a}{t}))\right]
\end{align}
\end{theorem}

This recursive formula is similar to that of the usual Bellman equation with KL regularization in ~\eqref{eq:max-ent-v-bellman}, yet they are different in two points.
First, the expectation is taken for the backward probability $\back$ in~\eqref{eq:bellman-back} whereas it is for the policy $\pi_\theta$ in~\eqref{eq:max-ent-v-bellman}.
Second, $\popV{t}$ in~\eqref{eq:bellman-back} is time-dependent while the equation in ~\eqref{eq:max-ent-v-bellman} is time-independent.
Although these differences exist, we can regard $\lambda(\pi)$ as a variant of the objective functions of entropy-regularized RL.
In summary, Theorem~\ref{thm:bellman-back} indicates that the diversity of the population induces KL-regularization, which accelerates the exploration.
COMM{SN}{We also note that this form of variational problem and KL-regularization is similar to trust region policy optimization~\cite{schulman2015trust}.

\section{Experimental Study}
To demonstrate that ARL unifies ZOO and POGA, we evaluated the performance of these algorithms for two setups.
The first one is tableau MDP, for which gradient estimation is more important than populational exploration.
The second one is Cart Pole problem, in which the objective has rugged landscape and thus populational exploration becomes more crucial than gradient estimation.

\subsection{Tableau MDP}
We consider the simplest case where there are two states $\Xset = \{x_0, x_1\}$ and two actions $\Aset = \{a_0, a_1\}$.
The initial state is $x_0$.
When the action is $a_1$, no state transition occurs.
When the action is $a_0$, the state transits to the other state.
The initial policy is $\pi(a) = 0.5$ for all $a \in \Aset$.
The reward is $1$ when the state is $x_0$ and $0$ otherwise.
The discount rate $\gamma$ and the time horizon $T$ are $0.9$ and $30$, respectively. 
Then, the optimal value of the objective becomes around $9.57$.
Since this problem has an unimodel objective, gradient estimation facilitates the optimization.
For all algorithms, the population size is fixed to $1000$ for comparison.

Figure~\ref{fig:tableu-mdp} shows the trajectories of the maximum cumulative rewards among the population at each iteration for the three algorithms.
The average (solid line) and standard deviation (shaded zone) of the cumulative reward are obtained by five independent experiments with different seeds.
We found that ZOO and ARL, which estimate gradient, can achieve the optimal value whereas POGA fails to achieve it.
Although the average of the cumulative reward of ZOO fluctuates around 9,  each trajectory of the cumulative reward of ZOO touches the optimal value.
Moreover, ARL has more robust and smooth learning trajectory than ZOO. Note that ZOO becomes more robust and stable if larger population size is used and learning rate is fine-tuned.
These results demonstrate that ARL can leverage gradient information as good as ZOO does to achieve a stable and prompt convergence to the optimal solution. 

\begin{figure}[t]
  \begin{center}
        \includegraphics[width=1.0\linewidth]{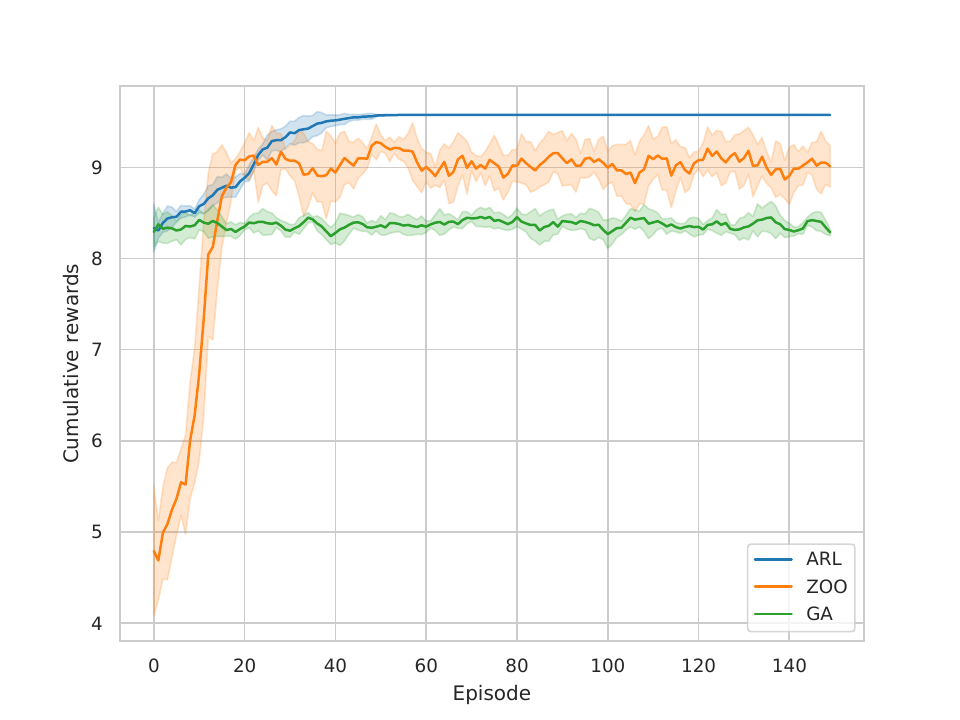}
   \end{center}
    \caption{Evaluation of ZOO (orange), POGA (green), and ARL (blue) for a tabeau MDP problem whose optimal cumulative reward is around $9.57$.
    The horizontal axis is the number of episodes whereas the vertical one is the cumulative reward of the best policy in the population at each iteration.
    Each solid line shows the average of the maximum cumulative reward  obtained by five independent trials. 
    The shaded zones around the curves are the standard deviation. 
    For visualization, we take a moving average of the cumulative reward with window size $5$.
    Both ARL and ZOO achieve the optimal value, whereas POGA fails to find the optimal policy. 
    We note that ZOO achieve the optimal value several times if we plot the trajectories of each trial.
    }
    \label{fig:tableu-mdp}
\end{figure}

\subsection{Cartpole}
\begin{figure}[t]
  \begin{center}
        \includegraphics[width=1.0\linewidth]{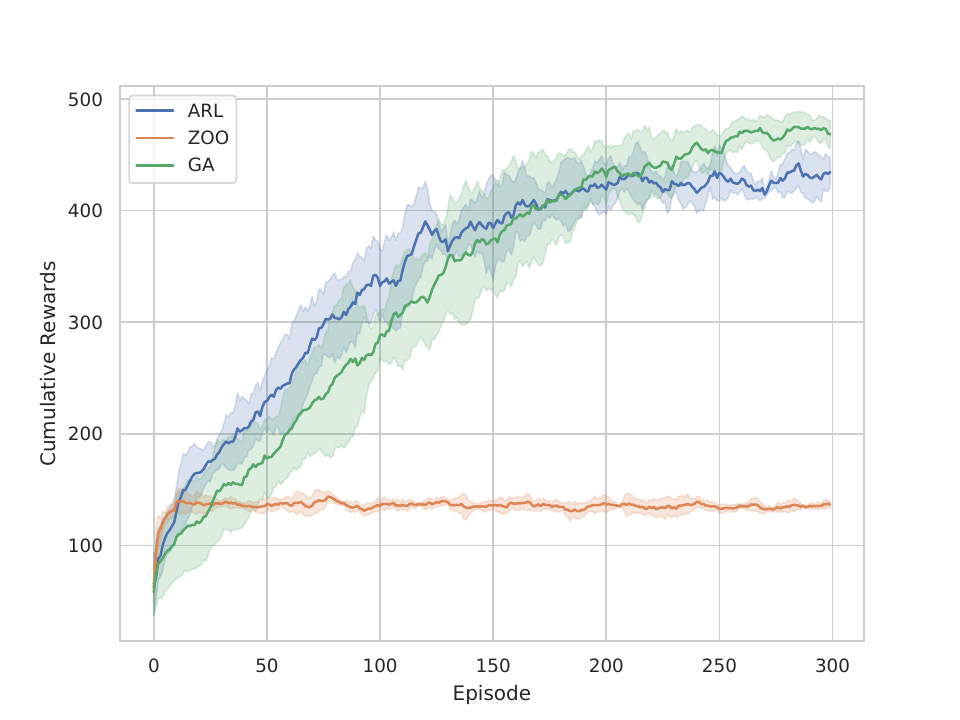}
   \end{center}
    \caption{Evaluation of ZOO (orange), POGA (green), and ARL (blue) for Cart Pols problem in OpenAI Gymnasium where the maximum cumulative reword is $500$. 
    The format of the figure follows those in Figure ~\eqref{fig:tableu-mdp} for ZOO and POGA.  
    For ARL, the average and variance of the trajectories are obtained from the four trials out of five, in which ARL succeeded to achieve almost the optimal value.
    For visualization, we take moving average whose window size is ten.
    }
    \label{fig:experiment}
\end{figure}
Next, we evaluate ZOO, POGA, and ARL using Cart Pole benchmark in the OpenAI Gymnasium~\cite{towers2024gymnasium}.
In this problem, a pole is attached to a friction-free cart and the objective is to keep it upright.
The available actions are to push the cart to the left ($a_{\mathrm{left}}$) or to the right ($a_{\mathrm{right}}$) with a constant power.
The simulation halts either when the pole falls or when the cart moves too far from the initial position.
The cumulative reward is the duration of time until the simulation ends, and thus the maximum cumulative reward is equal to the time horizon $T=500$.

We trained the following linear policy that chooses the direction of the action.
The state $x$ is a vector consisting of the four observables: cart position, cart velocity, pole angle, and pole angular velocity.
The policy is then defined by
\begin{align}
    \pi(a_\mathrm{left} \mid x) = \varsigma(\theta \cdot x),\\
    \varsigma(z) = 1 / (1 + e^{-z}),
\end{align}
where $\theta \in \mathbb{R}^4$ is the parameter.
The population size is $300$ for all algorithms. 
We generate a random $\theta$ for each agent in the population.

We conducted five independent trials with different seeds for the three algorithms.
Because this problem has a rugged objective function, learning trajectory can be stuck at a local maximum. 
Figure~\ref{fig:experiment} shows the trajectories of the maximum cumulative rewards among the population at each iteration where solid lines and shaded regions are the average and standard deviation of the cumulative reward computed by the five trials, respectively.
For all five trials, POGA could finally achieve the almost optimal value whereas ZOO could not escape from local maximums.
ARL demonstrates the behaviors in between them. For four out of the five trials, ARL could achieve almost optimal value while it failed to escape from a local maximum once. 
Figure~\ref{fig:experiment} shows the average trajectory of ARL computed from the four successful trials, which are close to those of POGA. 
This result indicates that ARL inherits the exploratory ability of POGA, enabling ARL to escape from local minima more likely even in the rugged objective landscape.

Overall, the numerical experiments demonstrate that ARL can attain the best of the two algorithms, POGA and ZOO, by integrating populational exploration and gradient estimation via the ancestral learning.

\section{Conclusion}
In this paper, we proposed ARL that unifies ZOO and GA.
Its theoretical basis was established by showing that agents in ARL can estimate the gradient from the ancestor's history via ancestral learning.
We also clarified that the exploratory search of ARL by population is linked to an implicit KL regularization of the objective function.
Numerical experiments demonstrated that ARL indeed combines the best of two algorithms, ZOO and POGA. 

In future work, we may improve ARL using a more efficient learning rule than ancestral learning.
As the ancestral learning is a variant of policy gradient, other methods (e.g. Q-learning) may be employed to learn from ancestral history.
To this end, we can exploit the similarity between the Bellman type recursion ~\eqref{eq:bellman-back} for the generalized V-function and the usual Bellman equation~\eqref{eq:max-ent-v-bellman}.
In addition, we may understand and evaluate the efficiency of existing algorithms in which population algorithms are incorporated with RL~\cite{khadka2018evolution, callaghan2023evolutionary}.
In any case, our theory can serve as a firm basis for bridging and unifying single-agent algorithms and population algorithms.

\section{Acknowledgements}
The authors thank Ignacio Madrid for discussions.
This research is supported by Toyota Konpon Research Institute, Inc., JST CREST JPMJCR2011, and JSPS KAKENHI Grant Numbers 24H01465 and 24H02148.

\bibliography{main}

\appendix
\onecolumn

\section*{Appendix}

\subsection{Population dynamics}
We derive the time evolution of $\suptime{p}{n}$ when the size of the population is infinite.
Let $\{\pi_i\}_{i=1,2,\dots, \popSize-1}$, $\Xpath_i$, and $\Apath_i$ be the policies in the population at the $n$-th iteration and their history of the states and actions, respectively.
Let the expected fraction $\suptime{p}{n+1}(\pi, \Xpath, \Apath \mid \mathcal{F}_n)$ of the policies whose parent's history of states and actions be $\Xpath$ and $\Apath$ at the next iteration.
Here, $\suptime{p}{n+1}$ is conditioned by the event $\mathcal{F}_n$ until the $n$-th iteration.
We note that $\suptime{p}{n+1}=\E_{\mathcal{F}_n}\left[ \suptime{p}{n+1}(\cdot \mid \mathcal{F}_n) \right] $.
By the definition of the algorithm, we have the following.
\begin{align}
    \suptime{p}{n+1}(\pi, \Xpath, \Apath \mid \mathcal{F}_n) = \frac{\exp({\beta R[\Xpath,\Apath]})}{\E_{\suptime{j}{n}}\left[\exp({\beta R[\Xpath,\Apath]})\right] } \suptime{j}{n}(\pi, \Xpath, \Apath),
\end{align}

where $j$ is the empirical distribution of the policies at time $n$ defined by
\begin{align}
    \suptime{j}{n}(\pi, \Xpath, \Apath ) =\frac{1}{\popSize} \sum_i \delta_{\pi, \pi_i} \delta_{\Xpath, \Xpath_i} \delta_{\Apath, \Apath_i},
\end{align}
and $\delta$ is the Dirac's delta.
When the size of the population is infinite, we have the following two approximations.
First, $\suptime{p}{n+1}(\cdot \mid \mathcal{F}_n)$ converges to a deterministic value independent of $\mathcal{F}_n$ and
\begin{align}
    \suptime{p}{n+1}(\pi, \Xpath, \Apath \mid \mathcal{F}_n) = \suptime{p}{n+1}(\pi, \Xpath, \Apath).
\end{align}
Second, $\suptime{j}{n}(\pi, \Xpath, \Apath)$ also converges to its expectation  due to the law of large number:
\begin{align}
    \suptime{j}{n}(\pi, \Xpath, \Apath) \approx \suptime{p}{n}(\pi) \forward^\pi[\Xpath, \Apath],
\end{align}
from the Markov property of MDP.
Here, $\forward^\pi$ is the forward probability for policy $\pi$ defined in \eqref{eq:def-forward}.

Using these approximations, we have the approximate time evolution of $\suptime{p}{n}$ as follows:
\begin{align}
    \label{eq:def-population-dynamics-approx-detail}
    \suptime{p}{n+1}(\pi, \Xpath, \Apath) =   \frac{\exp(\beta R[\Xpath, \Apath])}{\E_{\suptime{p}{n}(\pi')}\left[\E_{\forward^{\pi'}}\left[\exp(\beta R[\Xpath, \Apath])\right] \right]} \suptime{p}{n}(\pi)\forward^\pi[\Xpath, \Apath] .
\end{align}
In particular, we have
\begin{align}
    \label{eq:def-population-dynamics-approx}
    \suptime{p}{n+1}(\pi) =   \frac{\E_{\forward^{\pi}}\left[\exp(\beta R[\Xpath, \Apath])\right]}{\E_{\suptime{p}{n}(\pi')}\left[\E_{\forward^{\pi'}}\left[\exp(\beta R[\Xpath, \Apath])\right] \right]} \suptime{p}{n}(\pi)
\end{align}
by taking expectation on $\Xpath$ and $\Apath$.

\subsection{Proof that natural gradient with~\eqref{eq:ancestral_grad} is equivalent to \eqref{eq:ancestral_learning}}
The Fisher's metrics $I(\pi)$ for conditional probability $\pi(\cdot \mid x)$ is 
\begin{align}
    I(\pi)_{a, a'} = \delta_{a,a'} \pi(a \mid x)^{-1}.
\end{align}
Therefore, the natural gradient with the estimated gradient~\eqref{eq:ancestral_grad} is
\begin{align}
        \suptime{\pi_i}{n+1}(\cdot \mid x) \gets \suptime{\theta_i}{n} + \alpha  I(\suptime{\pi_{i}}{n}(\cdot \mid x))^{-1} \ancestralGrad \lambda(\pi_{\suptime{\pi_i}{n}}),
\end{align}
for all $x \in \Xset$.
Let us calculate the $a$-th element of the second term:
\begin{align}
    &(I(\suptime{\pi_{i}}{n}(\cdot \mid x))^{-1} \ancestralGrad \lambda(\pi_{\suptime{\pi_i}{n}}))_a\\
    &= \suptime{\pi_{i}}{n}(a \mid x) \times \frac{1}{T} \sum_{t=0}^{T} \nabla_{\suptime{\theta_i}{n}} \log \pi_{\suptime{\theta_i}{n}}(\suptime{a_{p(i)}}{t} \mid \suptime{x_{p(i)}}{t})\\
    &=  \suptime{\pi_{i}}{n}(a \mid x) \times \frac{1}{T} \sum_{t=0}^{T} \delta_{a, \suptime{a_{p(i)}}{t}}\frac{1}{\suptime{\pi_{i}}{n}(a \mid x)}\\
    &=  \frac{1}{T} \sum_{t=0}^{T} \delta_{a, \suptime{a_{p(i)}}{t}}\\
    &\propto  \suptime{\pi_B}{n}(a \mid x).
\end{align}
This update rule is equivalent to~\eqref{eq:ancestral_learning}.

\subsection{Proof of~\eqref{eq:def-back}}
This equation easily follows from~\eqref{eq:def-population-dynamics-approx-detail}.
Indeed, we have
\begin{align}
    \back[\Xpath, \Apath] = \frac{\suptime{p}{n+1}(\pi, \Xpath, \Apath)}{\suptime{p}{n+1}(\pi)} \propto \exp(\beta R[\Xpath, \Apath]) \forward^\pi[\Xpath, \Apath].
\end{align}

\subsection{Proof of Proposition~\ref{thm:theta-gradient-pop}}
\label{sec:proof-of-policy-gradient-lambda}
\begin{proof}
For simplicity, we denote $\suptime{\theta_i}{n}$ by $\theta$.
By definition,
\begin{align}
    &\exp[\beta \lambda(\pi_{\theta})] = \sum_{\Xpath, \Apath} e^{\beta R[\Xpath, \Apath]} \forward^{\pi_{\theta}}[\Xpath, \Apath].
\end{align}
By differentiating the both hand side by $\theta$, we have
\begin{align}
    \label{eq:prove-policy-gradient}
    \nabla_\theta  \exp[\beta \lambda(\pi_{\theta})]  &= \sum_{\Xpath, \Apath} e^{\beta R[\Xpath, \Apath]} \nabla_\theta  \forward^{\pi_{\theta}}[\Xpath, \Apath],
\end{align}
since $R[\Xpath, \Apath]$ is independent of $\theta$.
By~\eqref{eq:def-forward}, we have
\begin{align}
    & \nabla_\theta \forward^{\pi_{\theta}}[\Xpath, \Apath] \\
    &=\nabla_\theta \left[ \prod_t  \pi_\theta(\suptime{a}{t} \mid \suptime{x}{t}) \right],\\
    &= \sum_{t=0}^{T} \left[ \nabla_\theta   \pi_\theta(\suptime{a}{t} \mid \suptime{x}{t}) \right] \left[\prod_{s \neq t}  \pi_\theta(\suptime{a}{s} \mid \suptime{x}{s}) \right]
\end{align}
on the path where $\suptime{x}{t+1} = f(\suptime{x}{t}, \suptime{a}{t})$ for $t=0,1,\dots$.
By using the log derivative trick:
\begin{align}
   \nabla_\theta  \log \pi_\theta(a \mid x) = \frac{ \nabla_\theta  \pi_\theta(a \mid x)}{\pi_\theta(a \mid x)},
\end{align}
we have
\begin{align}
    &\nabla_\theta \forward^{\pi_{\theta}}[\Xpath, \Apath] \\
    &= \sum_{t=0}^{T} e^{\beta R[\Xpath, \Apath]} \left(\nabla_\theta  \log \pi_\theta(\suptime{a}{t} \mid \suptime{x}{t}) \right)  \forward^{\pi_{\theta}}[\Xpath, \Apath]\\
    &=\exp[\beta \lambda(\pi_{\theta})] \sum_t  \left(\nabla_\theta  \log \pi_\theta(\suptime{a}{t} \mid \suptime{x}{t}) \right) \back^{\pi_{\theta}}[\Xpath, \Apath].
\end{align}
By combining this equation to~\eqref{eq:prove-policy-gradient}, we have
\begin{align}
    \nabla_\theta \exp[\beta \lambda(\pi_{\theta})] = \beta \exp[\beta \lambda(\pi_{\theta})] \nabla_\theta \lambda(\pi_{\theta})=\exp[\beta \lambda(\pi_{\theta})] \E_{\back^{\pi_{\theta}}[\Xpath, \Apath]} \left[ \sum_{t=0}^{T} \nabla_\theta \log \pi_{\theta}(\suptime{a}{t}, \suptime{x}{t})  \right]
\end{align}
Thus, 
\begin{align}
    \nabla_\theta \lambda(\pi_\theta) =\frac{1}{\beta} \E_{\back^{\pi_{\theta}}[\Xpath, \Apath]} \left[ \sum_{t=0}^{T} \nabla_\theta \log \pi_{\theta}(\suptime{a}{t}, \suptime{x}{t})  \right]
\end{align}
\end{proof}

\subsection{Explicit formula for the conditioned backward probability (Eq.~\eqref{eq:truncated-back})}
We prove~\eqref{eq:truncated-back}.
Concretly, we prove the following two equations:
\begin{align}
    &\frac{\back[\Xpath, \Apath]}{\sum_{\suptime{\Xpath}{t+1:}, \suptime{\Apath}{t:}} \back[\Xpath, \Apath] }:= \back[\suptime{\mathbb{X}}{t+1:}, \suptime{\mathbb{A}}{t:} | \suptime{\Xpath}{:t}, \suptime{\Apath}{:t-1}]
    = \frac{\exp(\beta \suptime{R}{t:}[\suptime{\Xpath}{t:},\suptime{\Apath}{t:}]) \forward[\suptime{\Xpath}{t+1:},\suptime{\Apath}{t:} \mid \suptime{x}{t}]}{\exp(\beta \popV{t}(x) )},\\
    &\back[\suptime{\mathbb{X}}{t+1:}, \suptime{\mathbb{A}}{t:} | \suptime{\Xpath}{:t}, \suptime{\Apath}{:t-1}]= \back[\suptime{\Xpath}{t+1:},\suptime{\Apath}{t:} \mid \suptime{x}{t}].
\end{align}
The second equation means the independence on $\suptime{\Xpath}{:t-1}, \suptime{\Apath}{:t-1}$ given $\suptime{x}{t}$, which automatically follows from the first equation. 
Since 
\begin{align}
    \back[\mathbb{X}, \mathbb{A}] 
        =  e^{\beta \suptime{R}{0:t-1}[\suptime{\mathbb{X}}{0:t-1}, \suptime{\mathbb{A}}{0:t-1}]}e^{\beta \suptime{R}{t:}[\suptime{\mathbb{X}}{t:}, \suptime{\mathbb{A}}{t:}]}
         \forward[\suptime{\mathbb{X}}{1:t}, \suptime{\mathbb{A}}{0:t-1} ] \forward[\suptime{\mathbb{X}}{t+1:}, \suptime{\mathbb{A}}{t:}\mid\suptime{x}{t}], 
\end{align}
we have
\begin{align}
    \frac{\back[\Xpath, \Apath]}{\sum_{\suptime{\Xpath}{t+1:}, \suptime{\Apath}{t:}} \back[\Xpath, \Apath] }&=\frac{\left(e^{\beta \suptime{R}{0:t-1}[\suptime{\mathbb{X}}{0:t-1}, \suptime{\mathbb{A}}{0:t-1}]}\forward[\suptime{\mathbb{X}}{1:t}, \suptime{\mathbb{A}}{0:t-1} ]\right)\left(e^{\beta \suptime{R}{t:}[\suptime{\mathbb{X}}{t:}, \suptime{\mathbb{A}}{t:}]}
          \forward[\suptime{\mathbb{X}}{t+1:}, \suptime{\mathbb{A}}{t:}\mid\suptime{x}{t}]\right)}{\sum_{\suptime{\Xpath}{t+1:}, \suptime{\Apath}{t:}}\left(e^{\beta \suptime{R}{0:t-1}[\suptime{\mathbb{X}}{0:t-1}, \suptime{\mathbb{A}}{0:t-1}]}\forward[\suptime{\mathbb{X}}{1:t}, \suptime{\mathbb{A}}{0:t-1}  ]\right)\left(e^{\beta \suptime{R}{t:}[\suptime{\mathbb{X}}{t:}, \suptime{\mathbb{A}}{t:}]]}
          \forward[\suptime{\mathbb{X}}{t+1:}, \suptime{\mathbb{A}}{t:}\mid\suptime{x}{t}]\right)}\\
          =&\frac{e^{\beta \suptime{R}{t:}[\suptime{\mathbb{X}}{t:}, \suptime{\mathbb{A}}{t:}]}
          \forward[\suptime{\mathbb{X}}{t+1:}, \suptime{\mathbb{A}}{t:}\mid\suptime{x}{t}]}{\sum_{\suptime{\Xpath}{t+1:}, \suptime{\Apath}{t:}}e^{\beta \suptime{R}{t:}[\suptime{\mathbb{X}}{t:}, \suptime{\mathbb{A}}{t:}]}
          \forward[\suptime{\mathbb{X}}{t+1:}, \suptime{\mathbb{A}}{t:}\mid\suptime{x}{t}]}= \frac{\exp\left(\beta \suptime{R}{t:}[\suptime{\Xpath}{t:},\suptime{\Apath}{t:}] \right)\forward^{\pi}[\suptime{\Xpath}{t+1:},\suptime{\Apath}{t:} \mid \suptime{x}{t}]}{\exp [\beta \popV{t}(x)]},
\end{align}
where the last equality follows from the definition of $\popV{t}(x)$.

\subsubsection{Proof of~\eqref{eq:variational-representation}}
Let $\mathrm{P}[\Xpath, \Apath]$ be any probability distribution.
By using Jensen's inequality, we have
\begin{align}
    \popV{t}(x) 
    &= \frac{1}{\beta} \log \E_{\forward^{\pi}[\cdot|\suptime{x}{t}]} \left[ \exp\left(\beta \suptime{R}{t:}[\suptime{\Xpath}{t:},\suptime{\Apath}{t:}] \right) \right]\\
    &= \frac{1}{\beta} \log \sum_{\suptime{\Xpath}{t+1:}, \suptime{\Apath}{t:}} \forward^{\pi}[\suptime{\Xpath}{t+1:}, \suptime{\Apath}{t:}|\suptime{x}{t}]  \exp\left(\beta \suptime{R}{t:}[\suptime{\Xpath}{t:},\suptime{\Apath}{t:}] \right)\\
    &=  \frac{1}{\beta} \log \sum_{\suptime{\Xpath}{t+1:}, \suptime{\Apath}{t:}} \mathrm{P}[\suptime{\Xpath}{t+1:}, \suptime{\Apath}{t:}|\suptime{x}{t}] \frac{\forward^{\pi}[\suptime{\Xpath}{t+1:}, \suptime{\Apath}{t:}|\suptime{x}{t}]}{\mathrm{P}[\suptime{\Xpath}{t+1:}, \suptime{\Apath}{t:}|\suptime{x}{t}]}  \exp\left(\beta \suptime{R}{t:}[\suptime{\Xpath}{t:},\suptime{\Apath}{t:}] \right)\\
    &\ge \frac{1}{\beta} \sum_{\suptime{\Xpath}{t+1:}, \suptime{\Apath}{t:}} \mathrm{P}[\suptime{\Xpath}{t+1:}, \suptime{\Apath}{t:}|\suptime{x}{t}]  \left[ \log \frac{\forward^{\pi}[\suptime{\Xpath}{t+1:}, \suptime{\Apath}{t:}|\suptime{x}{t}]}{\mathrm{P}[\suptime{\Xpath}{t+1:}, \suptime{\Apath}{t:}|\suptime{x}{t}]}  +\beta \suptime{R}{t:}[\suptime{\Xpath}{t:},\suptime{\Apath}{t:}]    \right]\\
    &= \E_{\mathrm{P}[\suptime{\Xpath}{t+1:}, \suptime{\Apath}{t:}|\suptime{x}{t}]}\left[\suptime{R}{t:}[\suptime{\Xpath}{t:},\suptime{\Apath}{t:}]   \right] - \frac{1}{\beta} \KL{\mathrm{P}[\suptime{\Xpath}{t+1:},\suptime{\Apath}{t:} \mid \suptime{x}{t}]}{\forward^{\pi}[\suptime{\Xpath}{t+1:},\suptime{\Apath}{t:} \mid \suptime{x}{t}]}. 
\end{align}
By substituting $\mathrm{P}[\suptime{\Xpath}{t+1:},\suptime{\Apath}{t:} \mid \suptime{x}{t}]$ with backward probability 
\begin{align}
\back[\suptime{\Xpath}{t+1:},\suptime{\Apath}{t:} \mid \suptime{x}{t}]= \frac{\exp\left(\beta \suptime{R}{t:}[\suptime{\Xpath}{t:},\suptime{\Apath}{t:}] \right)\forward^{\pi}[\suptime{\Xpath}{t+1:},\suptime{\Apath}{t:} \mid \suptime{x}{t}]}{\exp [\beta \popV{t}(x)]},
\end{align}
we can see that the equality is achived in the above inequality by direct calculation.
Therefore, we have~\eqref{eq:variational-representation} and the maximizer is $\back[\suptime{\Xpath}{t+1:},\suptime{\Apath}{t:} \mid \suptime{x}{t}]$.

\subsubsection{Proof of Theorem~\ref{thm:bellman-back}}
For the first factor in~\eqref{eq:pop-v-by-back-prob}, we have
\begin{align}
    \E_{\back[\cdot\mid \suptime{x}{t}]}\left[\suptime{R}{t:}[\suptime{\Xpath}{t:}, \suptime{\Apath}{t:}]\right] 
    &= \E_{\back[\cdot\mid \suptime{x}{t}]}\left[\gamma^{t}r(\suptime{x}{t}, \suptime{a}{t}) +  \suptime{R}{t+1:}[\suptime{\Xpath}{t+1:}, \suptime{\Apath}{t+1:}]  \right]\\
    &=\E_{\back[\suptime{a}{t}\mid \suptime{x}{t}]}\left[\gamma^{t} r(\suptime{x}{t}, \suptime{a}{t}) + \E_{\back[\cdot\mid \suptime{x}{t+1}=f(\suptime{x}{t}, \suptime{a}{t})]}\left[\suptime{R}{t+1:}[\suptime{\Xpath}{1:},\suptime{\Apath}{t+1:}] \right] \right],
\end{align}
where $\back[\cdot\mid \suptime{x}{t}]=\back[\suptime{\Xpath}{t+1:}, \suptime{\Apath}{t:}\mid \suptime{x}{t}]$ and $\back[\cdot\mid \suptime{x}{t+1}]=\back[\suptime{\Xpath}{t+2:}, \suptime{\Apath}{t+1:}\mid \suptime{x}{t+1}]$. 
We also used 
\begin{align}
\back[\suptime{\Xpath}{t+1:}, \suptime{\Apath}{t:}\mid \suptime{x}{t}]&=\back[\suptime{\Xpath}{t+2:}, \suptime{\Apath}{t+1:}\mid \suptime{x}{t+1}, \suptime{a}{t}, \suptime{x}{t}]\back[\suptime{x}{t+1}, \suptime{a}{t:}\mid \suptime{x}{t}]\\
&=\back[\suptime{\Xpath}{t+2:}, \suptime{\Apath}{t+1:}\mid \suptime{x}{t+1}]\delta_{\suptime{x}{t+1}, f(\suptime{x}{t}, \suptime{a}{t})}\back[\suptime{a}{t:}\mid \suptime{x}{t}]
\end{align}

For the second factor in~\eqref{eq:pop-v-by-back-prob}, we have
\begin{align}
    &\KL{\back[\cdot\mid \suptime{x}{t}]}{\forward[\cdot\mid \suptime{x}{t}]} \\
    &=\E_{\back[\cdot\mid \suptime{x}{t}]}\left[\log \frac{\back[\suptime{\Xpath}{t+2:}, \suptime{\Apath}{t+1:}\mid \suptime{x}{t+1}]\delta_{\suptime{x}{t+1}, f(\suptime{x}{t}, \suptime{a}{t})}\back[\suptime{a}{t:}\mid \suptime{x}{t}]}{\forward[\suptime{\Xpath}{t+2:}, \suptime{\Apath}{t+1:}\mid \suptime{x}{t+1}]\delta_{\suptime{x}{t+1}, f(\suptime{x}{t}, \suptime{a}{t})}\forward[\suptime{a}{t:}\mid \suptime{x}{t}]} \right]\\
    &=\E_{\back[\suptime{a}{t}\mid \suptime{x}{t}]}\left[\log \frac{\back[\suptime{a}{t}\mid \suptime{x}{t}]}{\forward[\suptime{a}{t}\mid \suptime{x}{t}]}+\E_{\back[\cdot\mid \suptime{x}{t+1}=f(\suptime{x}{t}, \suptime{a}{t})]}\left[\log \frac{\back[\suptime{\Xpath}{t+2:}, \suptime{\Apath}{t+1:}\mid \suptime{x}{t+1}=f(\suptime{x}{t}, \suptime{a}{t})]}{\forward[\suptime{\Xpath}{t+2:}, \suptime{\Apath}{t+1:}\mid \suptime{x}{t+1}=f(\suptime{x}{t}, \suptime{a}{t})]}\right] \right]\\
    &= \E_{\back[\suptime{a}{t}\mid \suptime{x}{t}]}\left[\log \frac{\back[\suptime{a}{t}\mid \suptime{x}{t}]}{\pi(\suptime{a}{t}\mid \suptime{x}{t})}\right]+  \E_{\back[\suptime{x}{t}| \suptime{a}{t}]}\left[ \KL{\back[\suptime{\Xpath}{t+2:}, \suptime{\Apath}{t+1:} \mid \suptime{x}{t+1}]]}{\forward[\suptime{\Xpath}{t+2:}, \suptime{\Apath}{t+1:} \mid \suptime{x}{t+1}]} \right]
\end{align}
At the last equality, we used $\forward(\suptime{a}{t}\mid \suptime{x}{t})=\pi(\suptime{a}{t:}\mid \suptime{x}{t})$ and abbreviate the constraint $\suptime{x}{t+1}=f(\suptime{x}{t}, \suptime{a}{t})$ for notational simplicity.
Then,
\begin{align}
\popV{t}(\suptime{x}{t})=&\E_{\back[\suptime{a}{t}\mid \suptime{x}{t}]}\left[\gamma^{t} r(\suptime{x}{t}, \suptime{a}{t})-\frac{1}{\beta}\log \frac{\back[\suptime{a}{t}\mid \suptime{x}{t}]}{\forward[\suptime{a}{t}\mid \suptime{x}{t}]}\right. \\
&+\left. \E_{\back[\cdot\mid \suptime{x}{t+1}]}\left[\suptime{R}{t+1:}[\suptime{\Xpath}{t+1:},\suptime{\Apath}{t+1:}] \right]\right.\\
&\qquad\left.- \frac{1}{\beta}\KL{\back[\suptime{\Xpath}{t+2:}, \suptime{\Apath}{t+1:} \mid \suptime{x}{t+1}]]}{\forward[\suptime{\Xpath}{t+2:}, \suptime{\Apath}{t+1:} \mid \suptime{x}{t+1}]} \right]\\
=&\E_{\back[\suptime{a}{t}\mid \suptime{x}{t}]}\left[\gamma^{t} r(\suptime{x}{t}, \suptime{a}{t})-\frac{1}{\beta}\log \frac{\back[\suptime{a}{t}\mid \suptime{x}{t}]}{\forward[\suptime{a}{t}\mid \suptime{x}{t}]} + \popV{t+1}(f(\suptime{x}{t}, \suptime{a}{t}))\right],
\end{align}
where we used the constraint $\suptime{x}{t+1}=f(\suptime{x}{t}, \suptime{a}{t})$ at the last equality.
Then, Theorem~\ref{thm:bellman-back} was proven.

\subsection{Experiment Setting}
For numerical simulation, we used Intel(R) Core(TM) i7-8650U. 
We did not use any GPUs.
The size of random access memory is 16GB. The operating system is Windows Subsystem for Linux 2 (Linux version 5.15.153.1-microsoft-standard-WSL2).
The version of Python is 3.12.1. 
We install libraries via Poetry.
The setting file is released together with the source code.

\subsection{Extension to General MDP}
\label{sec:general-mdp}
In the main text, we considered the case that the state transition from $\suptime{x}{t}$ to $\suptime{x}{t+1}$ is deterministic given the action $\suptime{a}{t}$, i.e. $\suptime{x}{t+1}=f(\suptime{x}{t}, \suptime{a}{t})$. 
To relax this setup for stochastic transition from $\suptime{x}{t}$ to $\suptime{x}{t+1}$ in general MDP, we have to pay extra care to the source of stochasticity in the state transition, which does not matter for single-agent MDP but it does for multi-agent MDP.

There are two sources of stochasticity in the transition from $\suptime{x}{t}$ to $\suptime{x}{t+1}$. 
One is the stochasitcity coming from environment. 
Even if an agent takes action $a$ at state $x$, the next state can differ between at time $t$ and at time $t'$ because the environmental state is fluctuating over time. 
An example is controlling a glider.  Even if you tilt the control stick of a glider to the right ($a=\mathrm{right})$  in exactly the same way at the same spatial position $x$, the next position of the glider $x'$ would not always be the same at different time points because the direction and strength of the wind may not be the same at those points. 
In this example, the state of wind is the source of stochasticity from the environment.
The stochasticity of the transition may also originate from imperfect actions, e.g., if your action to tile the stick to right is not sufficiently precise, then the next position $x'$ would be different even if the previous state $x$ is exactly the same.
This is the stochasticity from the agent itself.

For the single agent MDP, these two sources need not be distinguished because we care only the stochastic law of the transition in the formulation of the single-agent MDP.
If we consider a population of agents as in POGA or ARL, however, the two situations can produce different outcomes. 
When the source is the environmental stochasticity and two agents who took the same action $a$ at time $t$ at the same state $x$, the next state should be the same. 
When the source is the agent's stochasticity, the next state can be different even if two agents took the same action $a$ at time $t$ at the same state $x$.

The results of the main text can be extended to the case where the source is environmental, where the realization $\Xpath$ of the state is the same for both agents if two agents take the same history of the actions. 
We call this property \textbf{Action-Dependent Determinism of States} (ADDS).
We introduce a lifted MDP, which explicitly formulates the situation where ADDS is satisfied, and the extended ARL algorithm in the next two sections.

\subsection{Lifted MDP with environmental stochasticity}
To extend POGA and ARL to general MDP, 
we introduce a lifted MDP, in which, a deterministic state transition rule from $(x', a)$ to $x$ is stochastically generated at each $t$ and each $n$. 
Specifically, let $\suptime{\mathcal{T}}{t}$ $(t=1,2,\dots)$ be a random map on $X \times A \rightarrow X$ such that
\begin{align}
    \suptime{\mathcal{T}}{t}(x', a) = x,
\end{align}
holds with probability $T(x | x',a)$ in an i.i.d. manner. 
Hereafter, we abbreviate the dependence on $n$ for notational simplicity.

Let us define the pair of initial state and state transition functions as
\begin{align}
    &\suptime{\mathscr{T}}{0} := (\suptime{x}{0}, \suptime{\mathcal{T}}{1}, \suptime{\mathcal{T}}{2},\dots),
\end{align}
where $\suptime{x}{0}$ is sampled from the initial state distribution $v$.
Similarly, we define
\begin{align}
    &\suptime{\mathscr{T}}{t} := (\suptime{x}{t}, \suptime{\mathcal{T}}{t+1}, \suptime{\mathcal{T}}{t+2},\dots).
\end{align}
Thus, once $\suptime{\mathscr{T}}{0}$ is sampled stochastically, the state transitions for all time points are deterministically fixed as in the main text. The following lemmas guarantee that the stochastic law in the original MDP is not altered:
\begin{lemma}
    \label{claim:construct-path-from-metapath}
    Let $\suptime{\mathcal{A}}{t}$ $(t=0,1,\dots)$ be a random variable on $X \rightarrow A$ such that $\suptime{\mathcal{A}}{t}(x) = a$ holds with probability $\pi(a | x)$ and define $\suptime{\mathscr{A}}{0} := (\suptime{\mathcal{A}}{0}, \suptime{\mathcal{A}}{1}, \dots)$.
    The, tuple $(\suptime{\mathscr{T}}{0}, \suptime{\mathscr{A}}{0})$ characterizes the realization of $\Xpath$ and $\Apath$.
\end{lemma}
\begin{proof}
    The initial state $\suptime{x}{0}$ is given by $\suptime{\mathscr{T}}{0}$.
    For other states and actions, we can recursively construct $\suptime{a}{t}$ and $\suptime{x}{t+1}$ from $\suptime{x}{t}$ by the following rule:
    \begin{align}
        &\suptime{a}{t} := \suptime{\mathcal{A}}{t}(\suptime{x}{t}),\\
        &\suptime{x}{t+1} := \suptime{\mathcal{{T}}}{t}(\suptime{x}{t}, \suptime{a}{t}).
    \end{align}
\end{proof}

\begin{lemma}
    The paths $\Xpath$ and $\Apath$ constructed by Lemma~\ref{claim:construct-path-from-metapath} follow the same distribution as the original MDP.
\end{lemma}
\begin{proof}
    From the construction in Lemma~\ref{claim:construct-path-from-metapath}, we know that $\suptime{x}{t;1}$ and $\suptime{a}{t}$ satisfies the same Markov property as the original MDP.
    Concretely, we know that
    \begin{align}
        &\mathrm{P}(\suptime{a}{t} \mid \suptime{\Xpath}{:t}, \suptime{\Apath}{:t-1}, \suptime{\mathscr{T}}{0}, \suptime{\mathscr{\Apath}}{0}) = \mathrm{P}(\suptime{a}{t} \mid \suptime{x}{t}, \suptime{\mathscr{T}}{0}, \suptime{\mathscr{A}}{0}),\\
        &\mathrm{P}(\suptime{x}{t+1} \mid \suptime{\Xpath}{:t}, \suptime{A}{:t-1}, \suptime{\mathscr{T}}{0}, \suptime{\mathscr{A}}{0}) = \mathrm{P}(\suptime{x}{t+1} \mid \suptime{x}{t},\suptime{\mathscr{T}}{0}, \suptime{\mathscr{A}}{0}).\\
    \end{align}
    Therefore, we can prove this claim by checking that the conditional distributions above are the same as the original MDP.
    For the initial state, we know that $\suptime{x}{0}$ follows the initial distribution $v$ from the definition of $\suptime{\mathscr{T}}{0}$.
    For $\suptime{a}{t}$, we know that $\suptime{a}{t}$ follows $\pi(\cdot \mid \suptime{x}{t})$ by the construction of $\suptime{\mathcal{A}}{t}$.
    For $\suptime{x}{t+1}$, we also know that $\suptime{x}{t}$ follows $T(\cdot \mid \suptime{x}{t}, \suptime{a}{t})$ by the same argument. 
\end{proof}

\subsection{Extended POGA and ARL for lifted MDP}
\begin{algorithm}[tb]
\caption{Unbiased Population Optimization via GA (POGA)}
\label{alg:unbiased_ga}
\begin{algorithmic}[1]
\State{Sample initial policies $\{\pi_{\suptime{\theta_i}{0}}\}_{i=1,2,\dots,\popSize-1}$.}
\For{$n=0, 1,\dots$}
    \State{Sample $\suptime{\mathscr{T}}{0}$.}
        \For{$i = 0,1,\dots, \popSize -1$}
            \State{Mutate policy $\pi_{\suptime{\theta_i}{n}}$ by adding small noise to the parameter and obtain $\pi_{\suptime{\theta_i'}{n}}$.}
            \State{Run simulation with $\suptime{\mathscr{T}}{0}$ for the mutated policy $\pi_{\suptime{\theta_i'}{n}}$ and observe the discounted cumulative reward $R_i$.}
            \State{Calculate fitness by $\suptime{f}{n}_i \gets \exp(\beta R_i)$.}
        \EndFor
        \State{Select each agent $\pi_{\suptime{\theta_{i}}{n+1}}$ in the next population independently from the mutated population $\{\pi_{\suptime{\theta_i'}{n}}\}_{i =0,1,\dots, \popSize-1}$ with probability proportional to $\suptime{f_i}{n}$.}
\EndFor
\end{algorithmic}
\end{algorithm}
\begin{algorithm}[tb]
\caption{Ancestral Reinforcement Learning}
\label{alg:unbised_arl}
\begin{algorithmic}[1]
\State{Sample initial policies $\{\suptime{\pi_{\theta_i}}{0}\}_{i =1,2,\dots, \popSize-1}$.}
\For{$n =0,1,\dots$}
    \State{Sample $\suptime{\mathscr{T}}{0}$.}
        \For{$i = 0,1,\dots, \popSize -1$}            
            \State{ $\suptime{\theta_i'}{n} \gets \mathrm{AncestralLearning}(\suptime{\theta_i}{n}, \Xpath, \Apath)$, where $\Xpath$ and $\Apath$ are the parent's history of states and action (Skip this step at $n=0$).}
            \State{Run simulation with  $\suptime{\mathscr{T}}{0}$ until time $T$ with policy $\pi_{\suptime{\theta_i'}{n}}$ and observe discounted cumulative reward $R_i$, the history $\Xpath = \{\suptime{x}{0}, \suptime{x}{1}, \dots \}$ of the states, and that of $\Apath = \{\suptime{a}{0}, \suptime{a}{1}, \dots \}$ of actions.}
        \EndFor
        \State{Calculate fitness as $\suptime{f}{n}_i \gets \exp(\beta R_i)$.}
        \State{Select each agent $\pi_{\suptime{\theta_{i}}{n+1}}$ for the next population independently from the mutated population $\{\{\pi_{\suptime{\theta_i'}{n}}\}_{i =0,1,\dots,\popSize-1}$ with probability proportional to $\suptime{f_i}{n}$.}
\EndFor
\end{algorithmic}
\end{algorithm}

Using the lifted MDP, we can construct extended POGA (Algorithm~\ref{alg:unbiased_ga}) and ARL (Algorithm~\ref{alg:unbised_arl}).
The key idea is that we sample $\suptime{\mathscr{T}}{0}$ at each iteration of the algorithm.
This $\suptime{\mathscr{T}}{0}$ is common for all agents and is used to determine the next state from the agent's action and previous state\footnote{Note that the sampling of action is conducted independently among population as in the main text.}.

It should be noted that the way we lift MDP does not limit applicability of ARL because what we have to do is to sample $\suptime{\mathscr{T}}{0}$ so that ASSD condition is guaranteed.
When we apply extended  POGA and ARL to numerical simulations, the sampling is realized  using the common sequence of pseudo-random numbers for all simulations at each iteration.
Moreover, ASSD condition can matter only when two agents can take exactly the same state and the same action at the same time with a sufficient likelihood.
This can happen in MDP with a discrete state space but could be practically negligible when the size of discrete state space is sufficiently large compared with the number of population to be used.
In addition, such an event is negligible in MDP with a continuous state space.
Thus, the lifted MPD is general for the practical purpose.

When we need to run extended  POGA and ARL in physical world using a population of physical agents, we cannot use common $\suptime{\mathscr{T}}{0}$.
However, the bias may not be problematic as physical MDPs usually have continuous state space and the number of available physical agents is severely restricted.

\subsection{Generalization of the result in the main text}
\subsubsection{Generalization of Lemma~\ref{thm:population-fitness}}
Since the forward probability is dependent on $\suptime{\mathscr{T}}{0}$, we denote $\forward[\Xpath, \Apath]$ as $\forward[\Xpath, \Apath \mid \suptime{\mathscr{T}}{0}]$.
However, when the dependency is clear from the context, we omit the conditioning.

To derive the population fitness for general MDP, we first generalize~\eqref{eq:def-pd} and~\eqref{eq:def-population-dynamics-approx-detail}.
Since $\suptime{p}{n}$ also depends on $\suptime{\mathscr{T}}{0}$, we denote $\suptime{p}{n}$ by $\suptime{p}{n}_{\suptime{\mathscr{T}}{0}}$.
\begin{lemma}
    Equation~\eqref{eq:def-population-dynamics-approx-detail} holds. Concretely, when conditioned by $\suptime{\mathscr{T}}{0}$, we have
    \begin{align}
        \label{eq:population-fitness-gradient-general-mdp}
            \suptime{p}{n+1}_{\suptime{\mathscr{T}}{0}}(\pi, \Xpath, \Apath) =   \frac{\exp(\beta R[\Xpath, \Apath])}{\E_{\suptime{p}{n}_{\suptime{\mathscr{T}}{0}}(\pi')}\left[\E_{\forward^{\pi'}[\Xpath, \Apath \mid \suptime{\mathscr{T}}{0}]}\left[\exp(\beta R[\Xpath, \Apath])\right] \right]} \suptime{p}{n}_{\suptime{\mathscr{T}}{0}}(\pi)\forward^\pi[\Xpath, \Apath \mid \suptime{\mathscr{T}}{0}] .
    \end{align}
\end{lemma}
\begin{proof}
    Almost the same proof for~\eqref{eq:def-population-dynamics-approx-detail} holds. We only mention the differences.
    Since $\suptime{j}{n}$ depends on $\suptime{\mathscr{T}}{0}$, we denote it by $\suptime{j}{n}_{\suptime{\mathscr{T}}{0}}$.
    Then, conditioning $\suptime{p}{n}, \suptime{p}{n+1}, \forward$, and $\suptime{j}{n}$ by $\suptime{\mathscr{T}}{0}$ in the proof of~\eqref{eq:def-population-dynamics-approx-detail}, we have the modified proof.
\end{proof}
In particular, we have an equation corresponding to ~\eqref{eq:def-pd}:
\begin{align}
    \suptime{p}{n+1}_{\suptime{\mathscr{T}}{0}}(\pi) =   \frac{\E_{ \forward^{\pi}[\Xpath, \Apath \mid \suptime{\mathscr{T}}{0}]}\left[\exp(\beta R[\Xpath, \Apath])\right]}{\E_{\suptime{p}{n}_{\suptime{\mathscr{T}}{0}}(\pi')}\left[\E_{ \forward^{\pi'}[\Xpath, \Apath \mid \suptime{\mathscr{T}}{0}]]}\left[\exp(\beta R[\Xpath, \Apath])\right] \right]} \suptime{p}{n}_{\suptime{\mathscr{T}}{0}}(\pi). 
\end{align}
From this equation, we defined a conditioned population fitness as follows:
\begin{align}
    \lambda(\pi \mid \suptime{\mathscr{T}}{0}) := \frac{1}{\beta} \log\E_{\forward^{\pi}[\Xpath, \Apath \mid \suptime{\mathscr{T}}{0}]}\left[ \exp({\beta R[\Xpath, \Apath])}\right],
\end{align}
We also define an averaged population fitness by
\begin{align}
    \label{eq:pd-general}
    \bar \lambda(\pi) := \E_{\suptime{\mathscr{T}}{0}}\left[ \lambda(\pi \mid \suptime{\mathscr{T}}{0}) \right],
\end{align}
where the expectation is taken for the distribution of $\suptime{\mathscr{T}}{0}$.

Let us prove Lemma~\ref{thm:population-fitness} for general MDP.
In the algorithm, $\suptime{\mathscr{T}}{0}$ is sampled independently at each iteration.
Let  $\suptime{\mathscr{T}}{0}_n$ be the value sampled at $n$-th iteration.
From~\eqref{eq:pd-general}, until the $N$-th iteration, the policy $\pi$ is amplified by the populational growth by the factor
\begin{align}
    \prod_{n=0}^{N-1} e^{\beta \lambda(\pi \mid \suptime{\mathscr{T}}{0}_n)}.
\end{align}
Therefore, if the iteration of the algorithm is sufficiently large, the average amplification by population growth at each time becomes 
\begin{align}
    \lim_{N \to \infty }\left(\prod_{n=0}^{N-1} e^{\beta \lambda(\pi \mid \suptime{\mathscr{T}}{0}_n)}\right)^{\frac{1}{N}} \approx e^{\beta \bar \lambda(\pi)}
\end{align}
due to the law of large number.
This equation implies that Lemma~\ref{thm:population-fitness} holds for general MDP.
In particular, we know that the objective function of POGA and ARL is $\bar \lambda(\pi)$ in general MDP.

\subsubsection{Generalization of Theorem~\ref{thm:ancestral-gradient}}
We first prove Theorem~\ref{alg:ancestral_learning} when conditioned by $\suptime{\mathscr{T}}{0}$.
After that, we prove the complete result by taking the average on $\suptime{\mathscr{T}}{0}$.

From~\eqref{eq:def-population-dynamics-approx-detail}, we define the backward probability by
\begin{align}
    \back^{\pi}[\Xpath, \Apath \mid \suptime{\mathscr{T}}{0}] :=\E\left[ j_{\pi}[\Xpath, \Apath] \right] \propto e^{\beta R[\Xpath, \Apath]} \forward^{\pi}[\Xpath, \Apath \mid \suptime{\mathscr{T}}{0}].
\end{align}
\begin{lemma}
    \begin{align}
    \label{eq:ancestral-grad-general-mdp}
        \nabla_\theta \lambda(\pi_\theta \mid \suptime{\mathscr{T}}{0}) =\frac{1}{\beta} \E_{\back^{\pi_{\theta}}[\Xpath, \Apath \mid \suptime{\mathscr{T}}{0}]} \left[ \sum_{t=0}^{T} \nabla_\theta \log \pi_{\theta}(\suptime{a}{t}, \suptime{x}{t})  \right].
    \end{align}
\end{lemma}
\begin{proof}
Basically, we can prove this lemma in the same way as Theorem~\ref{thm:ancestral-gradient}.
By conditioning $\lambda, \forward, \back$ by $\suptime{\mathscr{T}}{0}$, we can  prove the theorem.

For simplicity, we denote $\suptime{\theta_i}{n}$ by $\theta$.
By definition,
\begin{align}
    &\exp[\beta \lambda(\pi_{\theta} \mid \suptime{\mathscr{T}}{0})] = \sum_{\Xpath, \Apath} e^{\beta R[\Xpath, \Apath]} \forward^{\pi_{\theta}}[\Xpath, \Apath \mid \suptime{\mathscr{T}}{0}].
\end{align}
By differentiating the both hand side by $\theta$, we have
\begin{align}
    \nabla_\theta  \exp[\beta \lambda(\pi_{\theta} \mid \suptime{\mathscr{T}}{0})]  &= \sum_{\Xpath, \Apath} e^{\beta R[\Xpath, \Apath]} \nabla_\theta  \forward^{\pi_{\theta}}[\Xpath, \Apath \mid \suptime{\mathscr{T}}{0}],
\end{align}
since $R[\Xpath, \Apath]$ is independent of $\theta$.
By~\eqref{eq:def-forward}, we have
\begin{align}
    & \nabla_\theta \forward^{\pi_{\theta}}[\Xpath, \Apath \mid \suptime{\mathscr{T}}{0}] \\
    &=\nabla_\theta \left[ \prod_t  \pi_\theta(\suptime{a}{t} \mid \suptime{x}{t}) \right],\\
    &= \sum_{t=0}^{T} \left[ \nabla_\theta   \pi_\theta(\suptime{a}{t} \mid \suptime{x}{t}) \right] \left[\prod_{s \neq t}  \pi_\theta(\suptime{a}{s} \mid \suptime{x}{s}) \right]
\end{align}
on the path where $\suptime{x}{t+1} = \suptime{\mathcal{T}}{t}(\suptime{x}{t}, \suptime{a}{t})$ for $t=0,1,\dots$.
By using the log derivative trick:
\begin{align}
   \nabla_\theta  \log \pi_\theta(a \mid x) = \frac{ \nabla_\theta  \pi_\theta(a \mid x)}{\pi_\theta(a \mid x)},
\end{align}
we have
\begin{align}
    &\nabla_\theta \forward^{\pi_{\theta}}[\Xpath, \Apath \mid \suptime{\mathscr{T}}{0}] \\
    &= \sum_{t=0}^{T} e^{\beta R[\Xpath, \Apath]} \left(\nabla_\theta  \log \pi_\theta(\suptime{a}{t} \mid \suptime{x}{t}) \right)  \forward^{\pi_{\theta}}[\Xpath, \Apath \mid \suptime{\mathscr{T}}{0}]\\
    &=\exp[\beta \lambda(\pi_{\theta} \mid \suptime{\mathscr{T}}{0})] \sum_t  \left(\nabla_\theta  \log \pi_\theta(\suptime{a}{t} \mid \suptime{x}{t}) \right) \back^{\pi_{\theta}}[\Xpath, \Apath \mid \suptime{\mathscr{T}}{0}].
\end{align}
By combining this equation to~\eqref{eq:prove-policy-gradient}, we have
\begin{align}
    \nabla_\theta \exp[\beta \lambda(\pi_{\theta})] = \beta \exp[\beta \lambda(\pi_{\theta}\mid \suptime{\mathscr{T}}{0})] \nabla_\theta \lambda(\pi_{\theta}\mid \suptime{\mathscr{T}}{0})=\exp[\beta \lambda(\pi_{\theta}\mid \suptime{\mathscr{T}}{0})] \E_{\back^{\pi_{\theta}}[\Xpath, \Apath \mid \suptime{\mathscr{T}}{0}]} \left[ \sum_{t=0}^{T} \nabla_\theta \log \pi_{\theta}(\suptime{a}{t}, \suptime{x}{t})  \right].
\end{align}
Thus, 
\begin{align}
    \nabla_\theta \lambda(\pi_\theta\mid \suptime{\mathscr{T}}{0}) =\frac{1}{\beta} \E_{\back^{\pi_{\theta}}[\Xpath, \Apath \mid \suptime{\mathscr{T}}{0}]} \left[ \sum_{t=0}^{T} \nabla_\theta \log \pi_{\theta}(\suptime{a}{t}, \suptime{x}{t})  \right].
\end{align}
\end{proof}

Since the ancestral estimator of the gradient depends on $\suptime{\mathscr{T}}{0}$, we denote it by $\ancestralGrad \lambda(\pi_{\suptime{\theta_i}{n}} \mid \suptime{\mathscr{T}}{0})$.
We denote its average by
\begin{align}
    \ancestralGrad \bar \lambda(\pi_{\suptime{\theta_i}{n}}) := \E_{\suptime{\mathscr{T}}{0}}\left[\ancestralGrad \lambda(\pi_{\suptime{\theta_i}{n}} \mid \suptime{\mathscr{T}}{0}) \right].
\end{align}

\begin{theorem}
\begin{align}
     \E[\ancestralGrad \lambda(\pi_{\suptime{\theta_i}{n}} \mid \suptime{\mathscr{T}}{0})] \propto \nabla_{\suptime{\theta_i}{n}} \lambda(\pi_{\suptime{\theta_i}{n}}\mid \suptime{\mathscr{T}}{0}),
\end{align}
where the expectation is taken over the realization of the parent's history of states and actions.

In particular, by taking average on $\suptime{\mathscr{T}}{0}$, we have
\begin{align}
    \E[\ancestralGrad \bar \lambda(\pi_{\suptime{\theta_i}{n}})] \propto \nabla_{\suptime{\theta_i}{n}} \bar \lambda(\pi_{\suptime{\theta_i}{n}}),
\end{align}
where the expecatation is taken over both $\suptime{\mathscr{T}}{0}$ and the realization of the parent's history of states and actions

\end{theorem}
\begin{proof}
We prove the first equation.
By~\eqref{eq:ancestral-grad-general-mdp},
\begin{align}
    \nabla_\theta \lambda(\pi_\theta \mid \suptime{\mathscr{T}}{0}) =\frac{1}{\beta} \E_{\back^{\pi_{\theta}}[\Xpath, \Apath \mid \suptime{\mathscr{T}}{0}]} \left[ \sum_{t=0}^{T} \nabla_\theta \log \pi_{\theta}(\suptime{a}{t}, \suptime{x}{t})  \right].
\end{align}
By definition, have
\begin{align}
    \ancestralGrad \lambda(\pi_{\suptime{\theta_i}{n}}) = \E_{j_{i}^{}[\Xpath, \Apath \mid \suptime{\mathscr{T}}{0}]}\left[\sum_{t=0}^{T} \nabla_{\suptime{\theta_i}{n}} \log \pi_{\suptime{\theta_i}{n}}(\suptime{a}{t} \mid \suptime{x}{t})\right],
\end{align}
where we explicitly denote the dependency of $j_{i}$ on $\suptime{\mathscr{T}}{0}$.
Since $\E\left[j_{i}[\Xpath, \Apath \mid \suptime{\mathscr{T}}{0}] \right] = \back^{\pi_{\theta}}[\Xpath, \Apath \mid \suptime{\mathscr{T}}{0}]$, we have the first equation by combining these equations.
\end{proof}

\subsubsection{Generalization of Theorem~\ref{thm:bellman-back}}
We generalize $\popV{t}$ by
\begin{align}
    \popV{t}(x \mid \suptime{\mathscr{T}}{0}) := \frac{1}{\beta} \log \E_{\forward^{\pi}[\cdot\mid \suptime{x}{t}=x, \suptime{\mathscr{T}}{0}]} \left[ \exp\left(\beta \suptime{R}{t:}[\suptime{\Xpath}{t:},\suptime{\Apath}{t:}] \right) \right].
\end{align}
Then, we can generalize Theorem~\ref{thm:bellman-back} as follows.
\begin{theorem}
    \begin{align}
        \label{eq:bellman-back-general-mdp-detail}
        \popV{t}(\suptime{x}{t} \mid \suptime{\mathscr{T}}{0}) &=
        \E_{\back(\suptime{a}{t}\mid \suptime{x}{t} , \suptime{\mathscr{T}}{0})}\left[\gamma^{t} r(\suptime{x}{t}, \suptime{a}{t}) \right.\\
        &\left.-\frac{1}{\beta}\log \frac{\back(\suptime{a}{t}\mid \suptime{x}{t}, \suptime{\mathscr{T}}{0})}{\pi(\suptime{a}{t}\mid \suptime{x}{t})} + \popV{t+1}(\suptime{\mathcal{T}}{t}(\suptime{x}{t}, \suptime{a}{t}) \mid \suptime{\mathscr{T}}{0})\right].
    \end{align}
\end{theorem}
\begin{proof}
    We can prove thie theroem in the same way as Theorem~\ref{thm:bellman-back} by the following modification.
    First, we condition $\forward, \back$, and $\popV{t}$ by $\suptime{\mathscr{T}}{0}$ in the proof of Theorem~\ref{thm:bellman-back}.
    Second, replace $f(x,a)$ with $\mathcal{T}(x,a)$.
\end{proof}
We can prove the average version of the theorem.
We define
\begin{align}
    \bar{\popV{t}}(x) := \E_{\{\suptime{\mathcal{T}}{t'}\}_{t' \ge t}}\left[\popV{t}(x \mid\suptime{\mathscr{T}}{t})  \right].
\end{align}
\begin{theorem}
    \begin{align}
        \label{eq:bellman-back-general-mdp}
        \bar{\popV{t}}(\suptime{x}{t}) &=
        \E_{\suptime{\mathscr{T}}{0}}\left[\E_{\back(\suptime{a}{t}\mid \suptime{x}{t}, \suptime{\mathscr{T}}{0})}\left[\gamma^{t} r(\suptime{x}{t}, \suptime{a}{t})\right.\right.\\
        &\left.\left. \quad - \frac{1}{\beta}\log \frac{\back(\suptime{a}{t}\mid \suptime{x}{t}, \suptime{\mathscr{T}}{0})}{\pi(\suptime{a}{t}\mid \suptime{x}{t})} +\popV{t+1}(\suptime{\mathcal{T}}{t}(\suptime{x}{t}, \suptime{a}{t} \mid\suptime{\mathscr{T}}{0} )  \right]\right].
    \end{align}
\end{theorem}

\end{document}